%% file: main.tex
\documentclass{article} % For LaTeX2e
\usepackage[final]{colm2026_conference}

\usepackage{microtype}
\usepackage{hyperref}
\usepackage{url}
\usepackage{booktabs}
\usepackage{tcolorbox}
\usepackage{graphicx}
\usepackage{xcolor}
\usepackage{wrapfig}
\usepackage[dvipsnames]{xcolor}
\usepackage{array}
\usepackage{amsmath, amsthm, amssymb}
\usepackage{caption}
\usepackage{tocloft}        % customize Table of Contents (\cftsecleader, \cftsetindents, ...)
\usepackage[table]{xcolor}

\input{math_commands.tex}

\newtheorem{theorem}{Theorem}
\newtheorem{remark}{Remark}

\newcommand{\bx}{\mathbf{x}}
\newcommand{\bV}{\mathbf{V}}
\newcommand{\be}{\boldsymbol{\epsilon}}
\newcommand{\bd}{\boldsymbol{\delta}}
\newcommand{\bI}{\mathbf{I}}

\usepackage{lineno}

\definecolor{darkblue}{rgb}{0, 0, 0.5}
\hypersetup{colorlinks=true, citecolor=darkblue, linkcolor=darkblue, urlcolor=darkblue}

\title{Why Gaussian Diffusion Models Fail on Discrete Data and How to Prevent It?}

\author{Alexander Shabalin$^{1}$ \hspace{1em} Simon Elistratov$^{2}$ \hspace{1em} Viacheslav Meshchaninov$^{1}$ \\ \textbf{Ildus Sadrtdinov}$^{1}$\thanks{Shared senior authorship. Correspondence to: \texttt{ashabalin@constructor.university}} \hspace{1em} \textbf{Dmitry Vetrov}$^{1*}$\\
$^1$Constructor University, Bremen, Germany \\
$^2$Lomonosov Moscow State University, Moscow, Russia
}

\newcommand{\gray}[1]{\textcolor{gray}{#1}}
\newcommand{\tinycolorbox}[2]{\tikz[baseline=(a.base),inner sep=0pt]\node[fill=#1](a){#2};}

\begin{document}
\setcounter{tocdepth}{-1} % for the ToC of the Appendix

\ifcolmsubmission
\linenumbers
\fi

\maketitle

\begin{abstract}
  Diffusion models have become a standard approach for generative modeling in continuous domains, yet their application to discrete data remains challenging.
  We investigate why Gaussian diffusion models with the DDPM solver struggle to sample from discrete distributions that are represented as a mixture of delta-distributions in the continuous space. 
  Using a toy Random Hierarchy Model, we identify a critical sampling interval in which the density of noisified data becomes multimodal.
  In this regime, DDPM occasionally enters low-density regions between modes producing out-of-distribution inputs for the model and degrading sample quality.
  We show that existing heuristics, including self-conditioning and a solver we term q-sampling, help alleviate this issue.
  Furthermore, we demonstrate that combining self-conditioning with switching from DDPM to q-sampling within the critical interval improves generation quality on real data.
  We validate these findings across conditional and unconditional tasks in multiple domains, including text, programming code, and proteins.
\end{abstract}

\input{chapters/introduction.tex}

\input{chapters/background.tex}

\input{chapters/rhm.tex}

\input{chapters/unconditional_generation.tex}

\input{chapters/conditional_generation.tex}

\input{chapters/conclusion.tex}

\bibliography{colm2026_conference}
\bibliographystyle{colm2026_conference}

%%%%%%%%%%%%%%%%%%%%%%%%%%%%%%%%%%%%%%%%%%%%%%%%%%%%%%%%%%%%%%%%%%%%%%%%%%%%%%%
% APPENDIX
%%%%%%%%%%%%%%%%%%%%%%%%%%%%%%%%%%%%%%%%%%%%%%%%%%%%%%%%%%%%%%%%%%%%%%%%%%%%%%%

\newpage
\appendix
\onecolumn
\section*{Appendix}
\addtocontents{toc}{\protect\setcounter{tocdepth}{2}}
\addtocontents{toc}{\string\renewcommand{\string\cftsecfont}{\string\normalfont}}
\renewcommand{\cftsecleader}{\cftdotfill{\cftdotsep}}
\renewcommand{\contentsname}{} % Remove "Contents" title

\cftsetindents{section}{0em}{2em}
%\appendix

\tableofcontents
\vspace{1em}  
\hrule

%%%%%%%%%%%%%%%%%%%%%%%%%%%%%%%%%%%%%%%%%%%%%%%%%%%%%%%%%%%%%%%%%%%%%%%%%%%%%%%
\newpage
\appendix

\input{chapters/appendix.tex}

\end{document}

%% file: math_commands.tex
\usepackage{amsmath,amsfonts,bm}

\def\eqref#1{equation~\ref{#1}}
\def\1{\bm{1}}

\DeclareMathAlphabet{\mathsfit}{\encodingdefault}{\sfdefault}{m}{sl}
\SetMathAlphabet{\mathsfit}{bold}{\encodingdefault}{\sfdefault}{bx}{n}

\newcommand{\E}{\mathbb{E}}

\newcommand{\Var}{\mathrm{Var}}

%% file: chapters/introduction.tex
\section{Introduction}

In recent years, diffusion models~\citep{diffusion_orig} have achieved state-of-the-art generation quality across continuous domains, including images~\citep{stable_diffusion}, audio~\citep{wavegrad}, and video~\citep{video_diffusion}, but extending them to discrete domains such as text, graphs, or molecules remains challenging.
Consequently, autoregressive models~\citep{gpt3} remain the default for text generation, despite well-known drawbacks: left-to-right generation, one-token-at-a-time decoding, and the inability to revise earlier decisions.
Diffusion models offer a path to alleviating these limitations.

Existing diffusion-based text generation methods fall into two paradigms: discrete~\citep{d3pm, mdlm} and continuous~\citep{diffusion-lm, tencdm}.
In this work, we focus on continuous approaches because they do not suffer from the token factorization error \citep{lee2026flow, deschenaux2026the} and enable the reuse of techniques developed for continuous domains, such as classifier-free guidance~\citep{cls_free_guidance} and consistency distillation~\citep{luo2023latent}.

The most common baseline represents text as a sequence of token embeddings and applies Gaussian diffusion in the resulting embedding space.
However, since the underlying data distribution is discrete, this approach may be prone to errors, as we illustrate in Figure~\ref{fig:toy_example}.
In the continuous case, clean data lie on a manifold, and the density transitions smoothly to a standard Gaussian.
Small score-estimation errors rarely lead to irrecoverable failures, as they do not push the sampling trajectory off the manifold.
In the discrete case, by contrast, the density $p(\mathbf{x}_t)$ separates into distinct modes centered at data points as $t \to 0$, and score-estimation errors combined with time discretization can trap samples between modes, yielding incorrect generations.
While the training set is finite for both discrete and continuous domains, diffusion models on continuous data in the generalization regime recover the underlying distribution~\citep{consistency_dm, early_stopping_dm} and the data manifold~\citep{manifold_dm}, and thus do not exhibit similar issues.

\begin{figure}
  \centering
  \begin{tabular}{c}
    \includegraphics[width=0.7\textwidth]{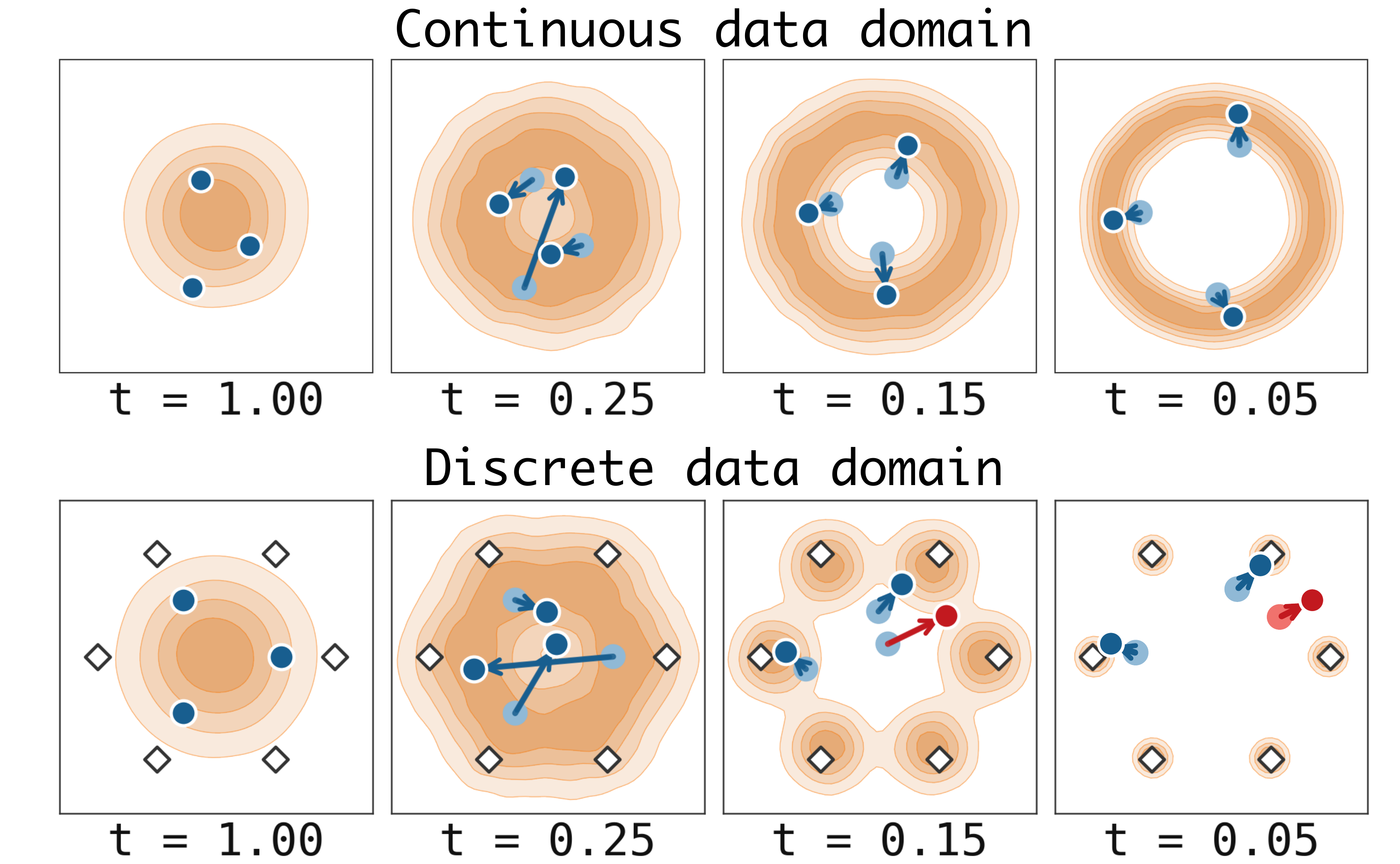}
  \end{tabular}
  \caption{Toy visualization of sampling with continuous diffusion models in continuous and discrete data domains. Orange contours visualize the density of noisified data $p(\textbf{x}_t)$. Blue points and arrows show in-distribution sampling trajectories, red --- out-of-distribution trajectories. White diamonds show original data points in the discrete setting.}
  \label{fig:toy_example}
\end{figure}

In this work, we systematically analyze when and why Gaussian diffusion fails on discrete data by studying the generation process in detail.
We argue that the most critical phase of sampling is the late stage, when the model's likelihood becomes multimodal and the model becomes highly sensitive to errors.
From this perspective, we explain why heuristics proposed in prior work, such as modified solvers~\citep{tess, smoothie} and self-conditioning~\citep{analog_bits}, improve generation of discrete data, despite lacking a comprehensive justification.
Our contributions are as follows:
\begin{itemize}
  \item We analyze the diffusion sampling process for discrete text generation.
  Using a simplified Random Hierarchy Model (RHM), we identify a \emph{critical time interval} in which the DDPM solver~\citep{ddpm} occasionally enters low-likelihood regions, producing out-of-distribution inputs for the model.
  \item We explain why existing heuristics, including \emph{self-conditioning} and the solver we call \emph{q-sampling}, mitigate this problem.
  We also explore \emph{Minimum Bayes-Risk}~\citep{mbr} decoding for unconditional text generation.
  \item We validate our findings on unconditional and conditional generation across several domains, including text, code, and proteins, and provide practical recommendations for improving generation quality, such as switching from DDPM to q-sampling and tuning the self-conditioning probability during training.
\end{itemize}

%% file: chapters/background.tex
\section{Background}

\paragraph{Text generation}
In natural language processing, unconditional text generation aims to draw a sample $\mathbf{w}$ from an unknown distribution $p(\mathbf{w})$, where $\mathbf{w} = [w^1, \dots, w^n]$ denotes a token sequence of variable length $n$.
For conditional text generation, the target distribution becomes $p(\mathbf{w} | \mathbf{y})$, where $\mathbf{y}$ is an observed conditioning variable; the goal is to produce text consistent with this condition.

\paragraph{Gaussian diffusion models}
Gaussian diffusion models~\citep{ddpm} consist of forward and reverse processes. Given a continuous data sample $\mathbf{x}_0 \sim p_{\mathrm{data}}$, a forward process produces a sequence of progressively noisier samples $\{\mathbf{x}_1, \dots, \mathbf{x}_T\}$, where $\mathbf{x}_t \sim \mathcal{N}(\sqrt{1 - \beta_t}\mathbf{x}_{t-1}, \beta_t\mathbf{I})$ and $\mathbf{x}_T \sim \mathcal{N}(0, \mathbf{I})$. Alternatively, $\mathbf{x}_t$ can be sampled directly from $\mathbf{x}_0$ with $q(\mathbf{x}_t | \mathbf{x}_0) = \mathcal{N}(\sqrt{\bar{\alpha}_t}\mathbf{x}_0, (1 - \bar{\alpha}_t)\mathbf{I})$, where $\alpha_t = 1 - \beta_t$ and $\bar{\alpha}_t = \prod_{i = 1}^t \alpha_i$.

To reverse this process, the model $\hat{\mathbf{x}}_{\theta}$ is trained to reconstruct a clean data sample $\mathbf{x}_0$ from the noisy latent $\mathbf{x}_t$ by minimizing the denoising objective:
\begin{equation}\label{eq:mse_loss}
    \mathcal{L}(\theta) = \mathbb{E}_{\mathbf{x}_0, t, q(\mathbf{x}_t | \mathbf{x}_0)} \Big[\|\mathbf{x}_0 - \hat{\mathbf{x}}_{\theta} (\mathbf{x}_t, t)\|^2\Big]
\end{equation}

\paragraph{Continuous Text Diffusion}
In continuous text diffusion, $\mathbf{x}_0$ is obtained by applying an encoder $\mathcal{E}$ to an input sequence $\mathbf{w}$.
Commonly, $\mathcal{E}$ is a simple embedding layer, although some works employ a transformer-based encoder instead~\citep{ld4lg, tencdm, cosmos}.
Unless otherwise stated, this paper uses the former approach.
To recover tokens from the latent, a linear decoder $\mathcal{D}$ is used.
In Appendix~\ref{app:related_work}, we discuss related works with their contributions in more detail.

To train the model, one can use pre-trained embeddings~\citep{sed, tencdm, smoothie} and optimize the loss in Eq.~\ref{eq:mse_loss}.
More commonly, embeddings are trained jointly with the model by minimizing the following objective~\citep{diffusion-lm, diffuseq}:
\begin{equation}\label{eq:full_loss}
    \mathcal{L}^{\mathrm{e2e}} = \mathbb{E}_{\mathbf{w}, t, \varepsilon \sim \mathcal{N}(0, \mathbf{I})} \Big[\big\|\mathcal{E}_\theta(\mathbf{w}) - \hat{\mathbf{x}}_{\theta} \big(\sqrt{\bar{\alpha}_t}\mathcal{E}_\theta(\mathbf{w}) + \sqrt{1 - \bar{\alpha}_t} \varepsilon, t\big)\big\|^2 - \log p_{\mathcal{D}}\big(\mathbf{w} | \mathcal{E}_\theta(\mathbf{w})\big)\Big]
\end{equation}

\paragraph{DDPM solver}

DDPM solver \citep{ddpm} is the most popular approach for reversing a diffusion process given the trained model $\hat{\mathbf{x}}_{\theta}$. It is defined as a Markov chain with Gaussian transitions starting at $p(\mathbf{x}_T) = \mathcal{N}(0, \mathbf{I})$:
\begin{gather}
p_\theta(\mathbf{x}_{t-1} | \mathbf{x}_t) =\mathcal{N}\left(\mathbf{x}_{t-1} ; \boldsymbol{\mu}_\theta(\mathbf{x}_t, t), \tilde{\beta}_t \mathbf{I} \right), \\
\text{where }\;
\boldsymbol{\mu}_\theta\left(\mathbf{x}_t, t\right) = \frac{\sqrt{\bar{\alpha}_{t-1}} \beta_t}{1-\bar{\alpha}_t} \hat{\mathbf{x}}_\theta(\mathbf{x}_t, t) + \frac{\sqrt{\alpha_t}\left(1-\bar{\alpha}_{t-1}\right)}{1-\bar{\alpha}_t} \mathbf{x}_t \quad \text{and} \quad \tilde{\beta}_t = \frac{1-\bar{\alpha}_{t-1}}{1-\bar{\alpha}_t} \beta_t
\end{gather}

%% file: chapters/rhm.tex
\section{Random Hierarchy Model}

\subsection{Setup description}

To analyze the dynamics of the DDPM solver on discrete data, we use the Random Hierarchy Model (RHM)~\citep{rhm} as a toy example, which is designed to mimic the hierarchical structure of text.
The dataset is generated by rules defined on a tree-structured graph with multiple levels, yielding sequences over a vocabulary of $V=16$ tokens with length $n=8$.
Of the $V^n$ possible sequences, only $N=16384$ satisfy the generation rules; the remaining $V^n - N\approx 4\cdot 10^9$ do not.
We refer to these as \emph{correct} and \emph{incorrect} sequences, respectively.
A detailed description of the data generation protocol is given in Appendix~\ref{app:rhm_description}.

To obtain the diffusion latent $\mathbf{x}_0$, we map each token to an embedding from a matrix $E\in\mathbb{R}^{V \times d}$, so that $\mathbf{x}_0 \in \mathbb{R}^{n \times d}$ with $d=2$.
We split the correct sequences into train and test sets in equal parts and optimize the loss in Eq.~\ref{eq:full_loss}.
We train the model until convergence, ensuring a memorization of $\sim0.5$, meaning the model generates samples from both splits.
As a primary metric we use \emph{correctness}, the rate of the correct sequences among generated. 

Since the number of correct sequences is finite, the data distribution can be written as $p_\mathrm{data}(\mathbf{x}) = \frac{1}{N}\sum_{\mathbf{x}_0} \delta(\mathbf{x} - \mathbf{x}_0)$, where $\delta$ is the Dirac delta function.
This allows us to compute the density of the noised latents $p(\mathbf{x}_t)$ analytically (see Eq.~\ref{eq:density} in Appendix~\ref{app:rhm_description}), as well as the optimal denoising prediction $\mathbf{x}^*_0(\mathbf{x}_t) = \mathbb{E}[\mathbf{x}_0 | \mathbf{x}_t] = \sum_{\mathbf{x}_0} \mathbf{x}_0 p(\mathbf{x}_0 | \mathbf{x}_t)$, which minimizes Eq.~\ref{eq:mse_loss}.

\subsection{Exploratory experiments}\label{sec:exploratory}

\begin{figure}
\centering
\includegraphics[width=0.9\textwidth]{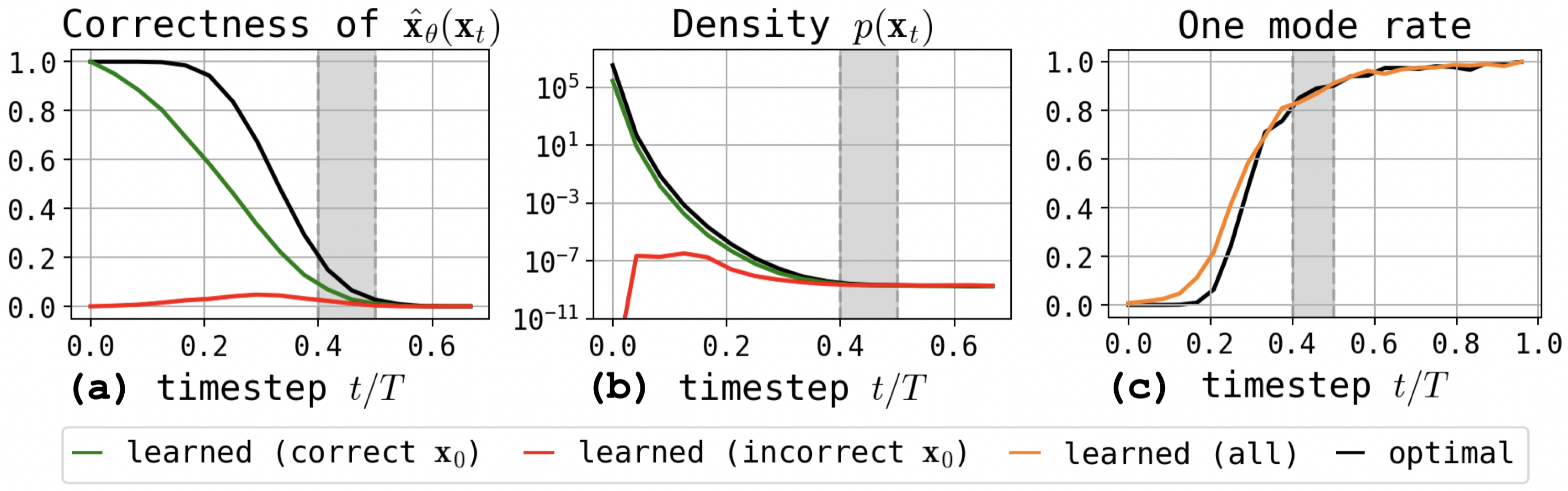}
\caption{\textbf{(a)} correctness of predictions $\hat{\mathbf{x}}_\theta(\mathbf{x}_t)$ and \textbf{(b)} mean probability density $p(\mathbf{x}_{t})$ along correct, incorrect and optimal sampling trajectories; \textbf{(c)} share of pairs of trajectories lying within the same mode (separately for all learned and optimal trajectories). Gray vertical stripes show the transition interval.}
\label{fig:ddpm-trajectories}
\end{figure}

To start, we exploit the advantages of the RHM setup and compare the trained diffusion model against the optimal prediction $\mathbf{x}^{*}_0(\mathbf{x}_t) = \mathbb{E}[\mathbf{x}_0 | \mathbf{x}_t]$.
As expected, DDPM with optimal predictions recovers sequences from the true data distribution, achieving a correctness of exactly $1$.
In contrast, DDPM with a trained model yields a correctness of $0.79$, indicating a deviation from the optimal prediction.
We therefore conclude that DDPM with a trained model is prone to errors, and proceed to identify when and why these errors occur.

To understand the failures of the sampling process, we partition DDPM trajectories into two groups based on whether they produce \emph{correct} or \emph{incorrect} samples at $t = 0$, and compare their dynamics against optimal prediction trajectories in Figure~\ref{fig:ddpm-trajectories}.
Figure~\ref{fig:ddpm-trajectories}a shows that for $t/T > 0.5$, all three trajectories exhibit identical zero correctness.
For $t/T < 0.4$, however, correct and optimal trajectories show growing correctness, while incorrect trajectories remain near zero for all $t$.
A similar pattern appears in the density plot in Figure~\ref{fig:ddpm-trajectories}b: all trajectories remain aligned until $t/T \approx 0.4$, after which correct trajectories trend upward while incorrect ones drop sharply, approaching zero density near $t/T = 0$.

Overall, correct and incorrect trajectories are practically indistinguishable for $t/T > 0.5$, but begin to diverge at $t/T \in [0.4, 0.5]$, where incorrect trajectories fall behind in both correctness and density.
We call this range the \emph{transition interval} and hypothesize that it marks a qualitative change in the shape of the distribution, making subsequent steps critical for accurate sampling.
We refer to $t/T \in [0, 0.4]$ as the \emph{critical interval} and focus our analysis on it to identify the factors that prevent DDPM from recovering correct sequences.

Since the data distribution is a mixture of delta functions and $p(\mathbf{x}_T)$ is unimodal, the distribution must transition from unimodal to multimodal at some point during sampling, as shown in Figure~\ref{fig:toy_example}.
We suppose that this transition happens in the \emph{transition interval}.
To verify this, we estimate the degree of multimodality at each sampling timestep by measuring whether pairs of samples $\mathbf{x}^1_t$ and $\mathbf{x}^2_t$ lie in the same mode (see details in Appendix~\ref{app:rhm_description}).

Figure~\ref{fig:ddpm-trajectories}c plots the fraction of pairs $(\mathbf{x}^1_t, \mathbf{x}^2_t)$ sharing the same mode.
At $t/T = 1$, all trajectories occupy a single mode, and this fraction decreases steadily as $t/T$ decreases towards $0.5$.
Between $0.1 < t/T < 0.4$, however, the fraction drops rapidly, and as $t/T$ approaches zero it becomes negligible, since each mode collapses to a single data sample.
Thus, around $t/T \approx 0.4$ a large number of modes begin to form, giving rise to \emph{low-density regions} in the diffusion space.
Incorrect trajectories may enter these regions and become trapped: since samples from such regions are unlikely to appear during the forward process, they act as out-of-distribution inputs for the model, which then fails to denoise them, producing incorrect sequences (we justify this empirically in Appendix \ref{app:midpoint_prediction}).
In Figure~\ref{fig:ddpm-trajectories}b, we show that incorrect trajectories indeed exhibit densities several orders of magnitude lower than correct ones.

\begin{tcolorbox}[colback=blue!5!white, colframe=blue!75!black]
\label{takeaway:1}
\textbf{Takeaway 1}: When modeling discrete data, the density of noised samples splits into multiple modes at some timestep, creating low-density regions between them. DDPM occasionally enters these regions, producing out-of-distribution inputs for the model and leading to incorrect generations.
\end{tcolorbox}

\subsection{Mitigating DDPM failures}\label{sec:mitigation}

\begin{figure}
\centering
\includegraphics[width=0.9\textwidth]{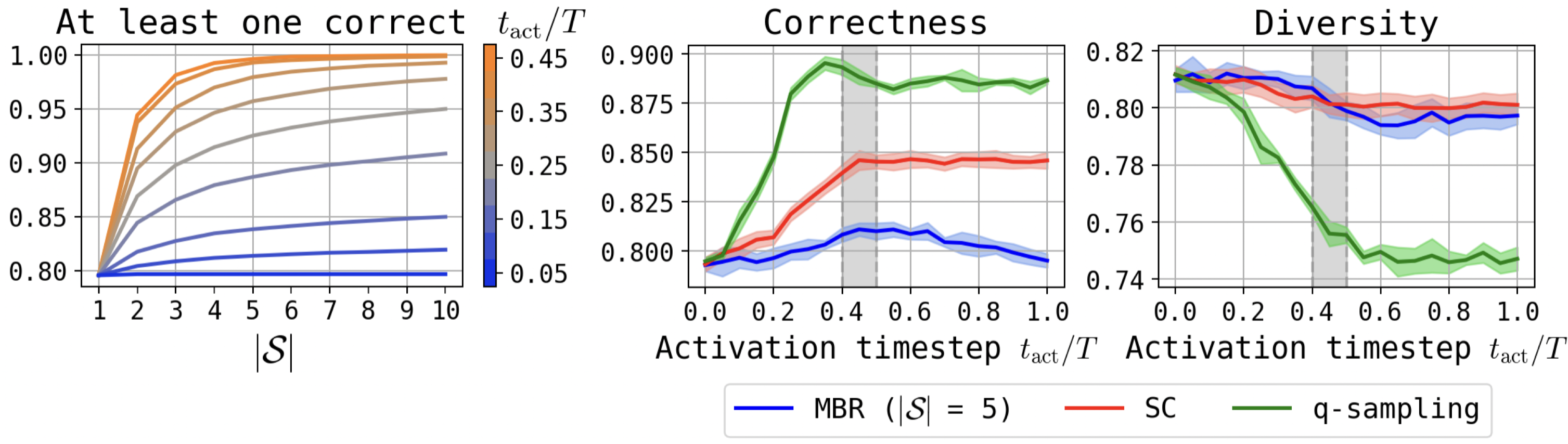}
\caption{\textbf{left:} ratio of candidate sets $\mathcal{S}$ with at least one correct sample vs. its size $|\mathcal{S}|$ for different activation timesteps $t_{\mathrm{act}}$; \textbf{right:} correctness and diversity of MBR with $|\mathcal{S}|=5$, self-conditioning (SC), and q-sampling vs. their activation timestep $t_{\mathrm{act}}$. Gray vertical stripes show the transition interval.}
\label{fig:switching}
\end{figure}

In this section, we discuss how the failures of DDPM can be mitigated.
Correct and incorrect trajectories differ only in sampled noise realizations and the resulting model predictions, so in principle, resampling the noise at a suitable timestep could shift an incorrect trajectory into a correct one.
To test this, we run DDPM up to $t = t_{\mathrm{act}} + 1$.
Then, at activation time $t = t_{\mathrm{act}}$, we branch the trajectory by sampling multiple noise realizations and completing the generation for each of them, which yields a set of samples $\mathcal{S}$.
Figure~\ref{fig:switching} (left) shows the fraction of such sets containing at least one correct sample, averaged over several repetitions.

For $t_{\mathrm{act}}/T = 0.05$, increasing $|\mathcal{S}|$ gives no improvement: at this stage the sequence is nearly fully formed, and resampling noise can no longer repair an incorrect outcome.
For larger activation timesteps, even a modest increase in $|\mathcal{S}|$ substantially improves correctness; at $t_{\mathrm{act}}/T \ge 0.4$, correctness approaches $1$ for sufficiently large $\mathcal{S}$.
These results suggest that DDPM sampling is reliable up to around $t/T = 0.4$, which matches with the start of the critical interval, and the steps below this threshold require more careful handling.
In the following paragraphs, we discuss three techniques that help mitigate errors in this regime.

\paragraph{Minimum Bayes-Risk (MBR)}
While branching trajectories improves correctness by up to $0.2$ (Figure~\ref{fig:switching}, left), in practice we cannot directly filter out incorrect samples, which necessitates a different selection criterion.
\emph{Minimum Bayes-Risk} (MBR) decoding~\citep{mbr} proposes such criterion for conditional text generation.
It selects the candidate with the highest expected utility relative to all other generated samples, for a given utility function (e.g., BLEU).
MBR has been applied to conditional text diffusion models~\citep{diffusion-lm, diffuseq}, but is not directly applicable to unconditional generation, where candidates are typically too diverse, causing MBR to favor the most ``average'' samples.

However, when candidates share a substantial part of their trajectory, namely the segment $t \in [t_{\mathrm{act}} + 1, T]$, they become much more similar, enabling MBR to filter incorrect samples without sacrificing diversity.
Figure~\ref{fig:switching} (right) shows correctness and diversity as a function of $t_{\mathrm{act}}$, using negative Hamming distance as the utility function.
Diversity is measured as the number of unique correct samples in a batch of $8192$, divided by the total number of correct samples.
MBR improves correctness when applied within the transition interval, but the gain diminishes for larger $t_{\mathrm{act}}$ as candidates become too diverse.
Crucially, it is possible to choose a timestep within the transition interval, such as $t_{\mathrm{act}}/T = 0.4$, where correctness improves while diversity remains high.
MBR examples are provided in Appendix~\ref{app:mbr_rhm}.

\paragraph{q-sampling}
While DDPM is the standard solver, several works~\citep{tess, smoothie} employ an alternative update rule:
\begin{align}\label{eq:q-sampling}
    \mathbf{x}_{t-1} = \sqrt{\bar{\alpha}_{t-1}} \hat{\mathbf{x}}_\theta(\mathbf{x}_{t}, t) + \sqrt{1 - \bar{\alpha}_{t-1}} \varepsilon, \quad \varepsilon \sim \mathcal{N}(0, \mathbf{I})
\end{align}
We refer to this as \emph{q-sampling}, since it is derived from the forward process
\(q(\mathbf{x}_{t} | \mathbf{x}_{0}) = \mathcal{N}(\sqrt{\bar{\alpha}_{t}} \mathbf{x}_0, (1 - \bar{\alpha}_{t}) \mathbf{I})\).
Unlike DDPM, q-sampling is not a valid solver in the sense that it does not reproduce the reverse process.
As shown in Appendix~\ref{app:prediction_variance}, q-sampling trajectories have lower marginal variance than those of DDPM, indicating that the two solvers effectively sample from different distributions.

Despite this, Figure~\ref{fig:switching} (right) shows that switching from DDPM to q-sampling at $t_{\mathrm{act}} / T > 0.2$ significantly improves correctness at the cost of reduced diversity.
Notably, activating q-sampling near the transition interval outperforms both pure DDPM ($t_{\mathrm{act}}/T = 0$) and pure q-sampling ($t_{\mathrm{act}}/T = 1$), suggesting that q-sampling does not inherit the pitfalls of DDPM and is particularly effective within the critical interval.

To understand this behavior, we compare single-step updates from DDPM and q-sampling.
We first generate multiple sampling trajectories using DDPM.
For each latent $\mathbf{x}_t$, we apply both solvers to obtain the corresponding $\mathbf{x}_{t-1}$ and compare the results in Figure~\ref{fig:q-sampling}.
Figure~\ref{fig:q-sampling}a shows that q-sampling improves correctness for samples that were incorrectly predicted at previous step for $t/T < 0.4$, while slightly reducing correctness for previously-correct predictions.
This is explained by Figure~\ref{fig:q-sampling}b: DDPM keeps $\mathbf{x}_{t-1}$ closer to $\mathbf{x}_t$, which is beneficial when the current prediction is correct, but harmful when it is not.
In the latter case, moving farther away, as q-sampling does, is preferable.
Figure~\ref{fig:q-sampling}c further shows that when the prediction is incorrect, q-sampling updates have higher density than DDPM updates, meaning the resulting $\mathbf{x}_{t-1}$ is less out-of-distribution.
Together, these properties explain why switching to q-sampling in the critical interval mitigates the failures of DDPM.

Finally, q-sampling should not be applied during the early stages of generation, where DDPM does not exhibit the above pitfalls and q-sampling underperforms in terms of prediction magnitude, as reflected in Figure~\ref{fig:q-sampling}d.

\begin{figure}
\centering
\includegraphics[width=1\textwidth]{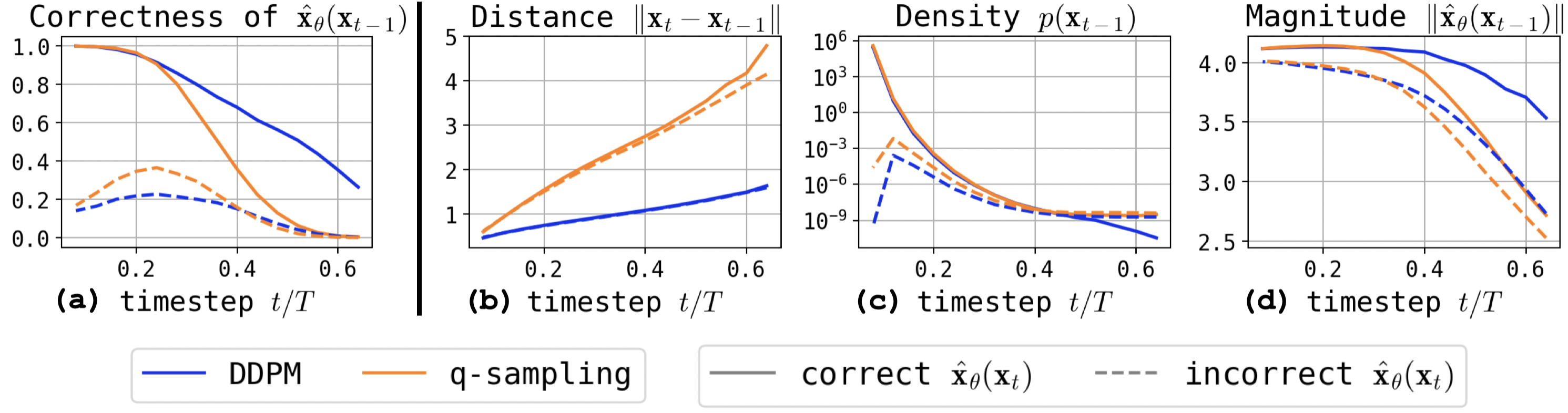}
\caption{Comparison of one-step updates for DDPM and q-sampling. For each plotted timestep $t$, we run DDPM from $t'=T$ to $t'=t$. Then, we divide the trajectories into correct and incorrect based on $\hat{\mathbf{x}}_{\theta}(\mathbf{x}_{t})$. The next step $t'=t-1$ is obtained with either DDPM or q-sampling.
The plots show \textbf{(a)} correctness of predictions $\hat{\mathbf{x}}_{\theta}(\mathbf{x}_{t-1})$, \textbf{(b)} mean distance $\|\mathbf{x}_{t} - \mathbf{x}_{t-1}\|$, \textbf{(c)} mean density $p(\mathbf{x}_{t-1})$, \textbf{(d)} mean magnitude of predictions $\|\hat{\mathbf{x}}_{\theta}(\mathbf{x}_{t-1})\|$.}
\label{fig:q-sampling}
\end{figure}

\paragraph{Self-conditioning (SC)}\label{par:sc}
Self-conditioning proposed in~\citet{analog_bits} is a technique that improves text generation quality~\citep{seqdiffuseq, ld4lg, tencdm} by providing the model with its own previous predictions during generation.
SC requires modifying the training procedure: with probability $p = 0.5$, the model receives its previous prediction as an additional input, $\hat{\mathbf{x}}^{t}_0 = \hat{\mathbf{x}}_{\theta}(\mathbf{x}_t, t, \operatorname{SG}(\bar{\mathbf{x}}^t_0))$, where $\bar{\mathbf{x}}^{t}_0 = \hat{\mathbf{x}}_{\theta}(\mathbf{x}_t, t, \mathbf{0})$ and $\operatorname{SG}$ denotes the stop-gradient operator. 
Otherwise, the SC input is set to zero, $\hat{\mathbf{x}}_0^{t} = \hat{\mathbf{x}}_{\theta}(\mathbf{x}_t, t, \mathbf{0})$.
During generation, the first prediction is made with SC input set to zero, and at all subsequent steps the previous prediction is passed as the SC input: $\hat{\mathbf{x}}^{t}_0 = \hat{\mathbf{x}}_{\theta}(\mathbf{x}_t, t, \hat{\mathbf{x}}^{t+1}_0)$.

Figure~\ref{fig:switching} (right) shows that SC improves correctness within the critical interval, though, like MBR and q-sampling, at the cost of reduced diversity.
Prior work attributes the benefits of SC to preventing information loss~\citep{analog_bits}, more efficient use of model capacity~\citep{cdcd}, and increased prediction confidence~\citep{tencdm}.
However, these explanations do not clarify what precisely happens when SC is applied to discrete data, nor why it is not used in continuous domain.
%image generation.
We argue that SC plays a role analogous to MBR and q-sampling: it helps generation trajectories avoid low-density regions within the critical interval.
Unlike those methods, SC achieves this by training the model to condition on its previous prediction, encouraging consecutive predictions to remain close to one another, as illustrated in Figure~\ref{fig:sc_trajectories}.

Figure~\ref{fig:sc_trajectories}a shows that within the critical interval, correct trajectories have smaller distances between consecutive predictions than incorrect ones, with SC reducing these distances further (Figure~\ref{fig:sc_trajectories}b).
Although SC deviates from optimal trajectories for $t/T > 0.3$, it more closely tracks optimal predictions than the baseline model at smaller timesteps, and eventually improves generation correctness (Figure~\ref{fig:sc_trajectories}c).
Therefore, SC mitigates the problem with unfavorable noise realizations by anchoring each prediction to the previous one, and can also be viewed as providing the model with two inputs: if one becomes out-of-distribution, the model can still rely on the other.

\begin{figure}
\centering
\includegraphics[width=0.9\textwidth]{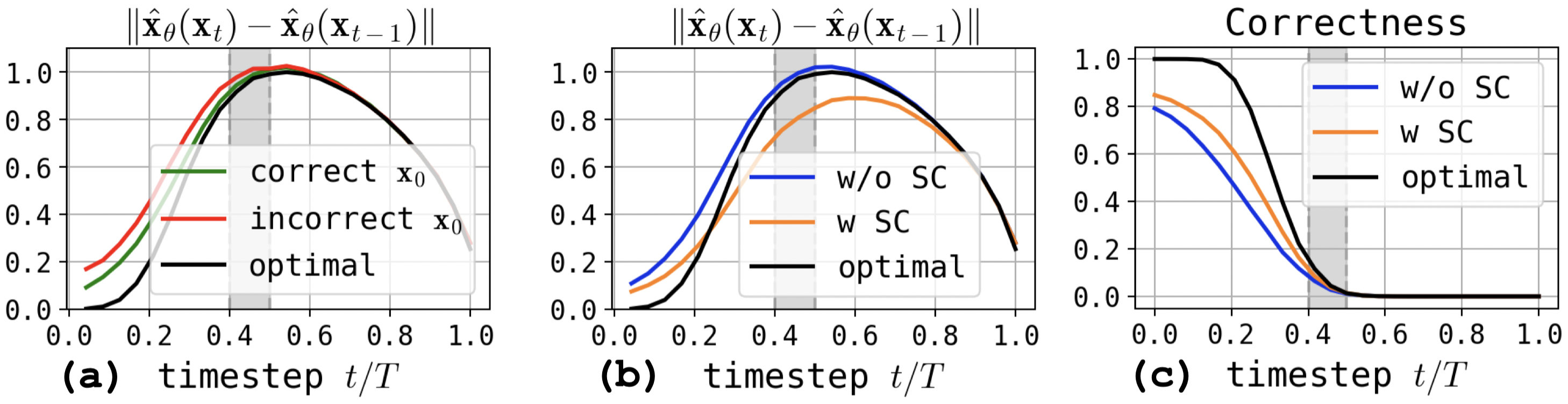}
\caption{\textbf{(a)} mean distance between predictions $\|\hat{\mathbf{x}}_{\theta}(\mathbf{x}_{t}) - \hat{\mathbf{x}}_{\theta}(\mathbf{x}_{t-1})\|$ for correct and incorrect trajectories without SC and for the optimal model; \textbf{(b)} mean distance between predictions $\|\hat{\mathbf{x}}_{\theta}(\mathbf{x}_{t}) - \hat{\mathbf{x}}_{\theta}(\mathbf{x}_{t-1})\|$ with and without SC; \textbf{(c)} correctness along the trajectory with and without SC. Gray vertical stripes show the transition interval.}
\label{fig:sc_trajectories}
\end{figure}

\begin{tcolorbox}[colback=blue!5!white, colframe=blue!75!black]
\label{takeaway:2}
\textbf{Takeaway 2}: 
MBR, q-sampling, and SC reduce DDPM generation errors within the critical interval by guiding trajectories away from low-density regions, but at the cost of discarding some correct trajectories and thus reducing diversity.
\end{tcolorbox}

In Appendices~\ref{app:image_diffusion} and~\ref{app:image_toy}, we confirm that the identified failure modes are specific to discrete distributions: SC and q-sampling provide no benefit in the continuous image domain.

%% file: chapters/unconditional_generation.tex
\section{Unconditional generation on real data}

In this section, we evaluate the impact of MBR, self-conditioning, and q-sampling on unconditional generation.
We experiment with Diffusion-LM~\citep{diffusion-lm}, which trains embeddings jointly with the diffusion model, and TEncDM~\citep{tencdm} in two latent space configurations: shallow BERT embeddings (TEncDM Emb) and context-dependent BERT encodings (TEncDM Enc).
Context-dependent encodings vary smoothly with the input, forming a continuous latent manifold rather than a set of isolated points. As a result, the probability mass spreads more uniformly across the latent space, leaving fewer low-density regions. Thus, we expect vanilla DDPM to be a stronger baseline in this setting.

\paragraph{Datasets}
We evaluate the models on ROCStories~\citep{rocstories} (98k examples), English Wikipedia~\citep{wikidump} (6.4M examples), and OpenWebText~\citep{openwebtext} (8M examples).
Full dataset descriptions are provided in Appendix~\ref{app:datasets}.

\paragraph{Metrics}
We assess performance using \textbf{perplexity} computed with GPT-2 Large~\citep{gpt2}, \textbf{diversity}~\citep{su2022a}, and \textbf{Mauve}~\citep{mauve}.
Perplexity and diversity are individually limited as metrics: perplexity favors low-diversity outputs, while diversity can be high for incoherent text and low when outputs are repetitive.
We therefore prioritize Mauve, which directly measures the proximity between the distributions of generated and reference texts, and encourages a balance between the perplexity and diversity.
Mauve is computed by comparing $1\,000$ generated and reference texts over $10$ random seeds for ROCStories and $20$ for other datasets, then averaging the results.
All metric values are provided in a $\mathrm{mean}_{\gray{\pm \mathrm{std}}}$ notation.
See other experimental details in Appendix~\ref{app:implementation}.

\begin{table}
\small
\centering
\setlength{\tabcolsep}{3.2pt}
\begin{tabular}{l|ccc|ccc}
\toprule
& \multicolumn{3}{c|}{\textbf{ROCStories}} & \multicolumn{3}{c}{\textbf{Wikipedia}} \\
\textbf{Model} & \textbf{Mauve} $\uparrow$ & \textbf{PPL} $\downarrow$ & \textbf{Div} $\uparrow$  & \textbf{Mauve} $\uparrow$ & \textbf{PPL} $\downarrow$ & \textbf{Div} $\uparrow$  \\
\midrule
Diffusion-LM & $16.2_{\gray{\pm 2.9}}$ & $148.9_{\gray{\pm 2.7}}$ & $\textbf{33.7}_{\gray{\pm 0.4}}$ & $1.2_{\gray{\pm 0.2}}$ & $419.8_{\gray{\pm 9.3}}$ & $49.5_{\gray{\pm 0.6}}$ \\
\, + SC$_{0.5}$ & $71.3_{\gray{\pm 3.8}}$ & $73.8_{\gray{\pm 0.8}}$ & $30.2_{\gray{\pm 0.3}}$ & $12.1_{\gray{\pm 1.9}}$ & $300.2_{\gray{\pm 6.7}}$ & $\textbf{58.7}_{\gray{\pm 0.4}}$ \\
\, + QS$_{0.3}$ & $72.7_{\gray{\pm 4.3}}$ & $42.7_{\gray{\pm 0.4}}$ & $17.4_{\gray{\pm 0.2}}$ &$5.8_{\gray{\pm 0.8}}$ & $166.6_{\gray{\pm 3.3}}$ & $40.0_{\gray{\pm 0.6}}$ \\
\, + QS$_{0.3}$ + SC$_{0.5}$ & $\textbf{81.6}_{\gray{\pm 2.7}}$ & $\textbf{41.0}_{\gray{\pm 0.6}}$ & $22.9_{\gray{\pm 0.3}}$ & $36.4_{\gray{\pm 4.4}}$ & $\textbf{114.2}_{\gray{\pm 2.4}}$ & $48.4_{\gray{\pm 0.4}}$ \\
\, + SC$_{0.95}$ & $75.2_{\gray{\pm 2.4}}$ & $59.5_{\gray{\pm 1.1}}$ & $29.0_{\gray{\pm 0.4}}$ & $27.8_{\gray{\pm 3.2}}$ & $197.4_{\gray{\pm 4.8}}$ & $56.6_{\gray{\pm 0.5}}$ \\
\, + QS$_{0.2}$ + SC$_{0.95}$ & $80.9_{\gray{\pm 3.0}}$ & $50.1_{\gray{\pm 0.4}}$ & $27.1_{\gray{\pm 0.3}}$ & 
$\textbf{44.8}_{\gray{\pm 4.7}}$ & $142.8_{\gray{\pm 3.2}}$ & $54.6_{\gray{\pm 0.5}}$ \\
\midrule
TEncDM Emb & $12.5_{\gray{\pm 1.4}}$ & $131.4_{\gray{\pm 2.7}}$ & $\textbf{37.4}_{\gray{\pm 0.4}}$ & $3.1_{\gray{\pm 0.4}}$ & $541.7_{\gray{\pm 10.5}}$ & $\textbf{63.4}_{\gray{\pm 0.5}}$ \\
\, + SC$_{0.5}$ & $56.1_{\gray{\pm 3.4}}$ & $42.4_{\gray{\pm 0.6}}$ & $24.7_{\gray{\pm 0.3}}$ & $44.5_{\gray{\pm 4.6}}$ & $138.2_{\gray{\pm 2.3}}$ & $50.2_{\gray{\pm 0.4}}$ \\
\, + QS$_{0.4}$ & $41.1_{\gray{\pm 6.8}}$ & $71.6_{\gray{\pm 1.2}}$ & $30.5_{\gray{\pm 0.2}}$ & $10.1_{\gray{\pm 2.0}}$ & $255.4_{\gray{\pm 7.5}}$ & $53.3_{\gray{\pm 0.6}}$ \\
\, + QS$_{0.4}$ + SC$_{0.5}$ & $62.7_{\gray{\pm 3.5}}$ & $\textbf{28.6}_{\gray{\pm 0.2}}$ & $19.5_{\gray{\pm 0.3}}$ & $\textbf{53.5}_{\gray{\pm 3.5}}$ & $86.8_{\gray{\pm 1.9}}$ & $44.5_{\gray{\pm 0.5}}$ \\
\, + SC$_{0.95}$ & $77.7_{\gray{\pm 3.5}}$ & $36.0_{\gray{\pm 0.7}}$ & $27.8_{\gray{\pm 0.4}}$ & $46.5_{\gray{\pm 4.9}}$ & $131.6_{\gray{\pm 2.9}}$ & $51.7_{\gray{\pm 0.6}}$ \\
\, + QS$_{0.5}$ + SC$_{0.95}$ & $\textbf{80.2}_{\gray{\pm 4.4}}$ & $29.5_{\gray{\pm 0.3}}$ & $25.7_{\gray{\pm 0.2}}$ & 
 $51.2_{\gray{\pm 4.8}}$ & $\textbf{79.3}_{\gray{\pm 2.0}}$ & $45.6_{\gray{\pm 0.5}}$ \\
\midrule
TEncDM Enc & $67.1_{\gray{\pm 4.6}}$ & $57.1_{\gray{\pm 0.9}}$ & $\textbf{32.2}_{\gray{\pm 0.3}}$ & $70.1_{\gray{\pm 4.2}}$ & $236.6_{\gray{\pm 5.5}}$ & $\textbf{63.2}_{\gray{\pm 0.5}}$ \\
\, + SC$_{0.5}$ & $81.1_{\gray{\pm 2.8}}$ & $29.0_{\gray{\pm 0.4}}$ & $26.4_{\gray{\pm 0.5}}$ & $87.3_{\gray{\pm 3.0}}$ & $94.0_{\gray{\pm 1.8}}$ & $55.2_{\gray{\pm 0.3}}$ \\
\, + QS$_{0.4}$ & $76.9_{\gray{\pm 3.1}}$ & $29.9_{\gray{\pm 0.2}}$ & $24.3_{\gray{\pm 0.2}}$ & $88.2_{\gray{\pm 2.2}}$ & $93.6_{\gray{\pm 1.7}}$ & $53.5_{\gray{\pm 0.4}}$ \\
\, + QS$_{0.4}$ + SC$_{0.5}$ & $82.1_{\gray{\pm 2.5}}$ & $29.1_{\gray{\pm 0.4}}$ & $26.6_{\gray{\pm 0.4}}$ & $88.4_{\gray{\pm 2.4}}$ & $79.1_{\gray{\pm 1.0}}$ & $53.3_{\gray{\pm 0.4}}$ \\
\, + SC$_{0.95}$ & $\textbf{86.8}_{\gray{\pm 2.9}}$ & $29.6_{\gray{\pm 0.3}}$ & $31.4_{\gray{\pm 0.5}}$ & $\textbf{90.1}_{\gray{\pm 2.1}}$ & $92.6_{\gray{\pm 1.6}}$ & $57.3_{\gray{\pm 0.3}}$ \\
\, + QS$_{0.5}$ + SC$_{0.95}$ & $85.8_{\gray{\pm 2.3}}$ & $25.5_{\gray{\pm 0.3}}$ & $29.9_{\gray{\pm 0.2}}$ & $\textbf{89.6}_{\gray{\pm 2.4}}$ & $\textbf{59.9}_{\gray{\pm 0.8}}$ & $52.6_{\gray{\pm 0.4}}$ \\
\bottomrule
\end{tabular}
\caption{Impact of q-sampling (QS) and self-conditioning (SC) on ROCStories and Wikipedia. SC trained with probability $p$ is denoted SC$_{p}$; QS with optimal $t_{\mathrm{act}}$ is denoted QS$_{t_{\mathrm{act}}/T}$.} 
\label{tab:q-sampling_sc}
\end{table}

\paragraph{Minimum Bayes-Risk}
As discussed above, MBR can be applied to unconditional generation by branching the trajectory at an activation timestep $t_\mathrm{act}$ within the transition interval, at which the coarse structure of the sample has been determined and only fine details are to refine\footnote{The size of the critical interval, and thus the optimal value of $t_\mathrm{act}$, depends on the noise schedule.}.
Table~\ref{tab:mbr} shows that this transfers to text generation.
Tuning $t_\mathrm{act}$ on ROCStories with
standard parameters (BLEU utility function and $|\mathcal{S}| = 10$~\citep{diffusion-lm, diffuseq}) improves both Mauve and perplexity across all three diffusion models.
This confirms that errors arising in the critical interval can be mitigated by resampling noise realizations.
More
\begin{wraptable}{r}{0.46\textwidth}
%\vspace{-1em}
\small
\centering
\setlength{\tabcolsep}{2.8pt}
\begin{tabular}{l|ccc}
\toprule
\textbf{Method} & \textbf{Mauve} $\uparrow$ & \textbf{PPL} $\downarrow$ & \textbf{Div} $\uparrow$ \\
\midrule
Diffusion-LM & $16.2_{\gray{\pm 3.0}}$ & $148.9_{\gray{\pm 2.7}}$ & $33.7_{\gray{\pm 0.4}}$ \\
\, + MBR$_{0.3}$ & $20.8_{\gray{\pm 3.5}}$ & $136.3_{\gray{\pm 1.8}}$ & $32.2_{\gray{\pm 0.3}}$  \\
TEncDM Emb & $15.7_{\gray{\pm 2.9}}$ & $123.2_{\gray{\pm 1.8}}$ & $36.6_{\gray{\pm 0.3}}$ \\
\, + MBR$_{0.3}$ & $18.9_{\gray{\pm 1.8}}$ & $101.5_{\gray{\pm 1.9}}$ & $31.6_{\gray{\pm 0.5}}$ \\
TEncDM Enc & $70.1_{\gray{\pm 3.2}}$ & $54.2_{\gray{\pm 1.1}}$ & $31.2_{\gray{\pm 0.4}}$ \\
\, + MBR$_{0.3}$ & $72.3_{\gray{\pm 2.0}}$ & $48.0_{\gray{\pm 44.5}}$ & $28.3_{\gray{\pm 0.3}}$ \\
\bottomrule
\end{tabular}
\caption{MBR on ROCStories ($|\mathcal{S}| = 10$).}
\label{tab:mbr}
\vspace{-2em}
\end{wraptable}
detailed analisys of MBR with examples is provided in Appendix~\ref{app:mbr_roc}.
While MBR improves quality, it drastically slows generation and thus, we omit it from further experiments.

\paragraph{Self-conditioning}
We apply SC along the entire sampling trajectory, as this matches the quality of restricting SC to the critical interval while avoiding an additional $t_{\mathrm{act}}$ hyperparameter.
As confirmed by our experiments in Table~\ref{tab:q-sampling_sc}, SC substantially improves generation quality across all three models.
The improvement is most pronounced for embedding-based models (Diffusion-LM and TEncDM Emb), which completely fail without SC but reach competitive performance with it.
The occasional increase in diversity with SC is a metric artifact: without SC, generated texts are often of such low quality (e.g., containing repetitions) that they suppress the diversity score.
For TEncDM Enc, the Mauve improvement is notable but less dramatic, consistent with our intuition that a more continuous latent space provides a stronger baseline.
Overall, as in the RHM setting, SC substantially improves perplexity at some cost to diversity, yielding a better Mauve trade-off.

Prior work universally uses the default SC probability $p=0.5$, meaning previous predictions are used only half the time during training.
We propose increasing $p$ closer to $1$: since the model conditions on its previous prediction at every step during sampling except the first, a higher $p$ better matches the inference procedure and encourages stronger reliance on self-conditioning.
Our results with $p=0.95$ confirm that this improves generation quality with only a $1.13\times$ increase in training time over the $p=0.5$ baseline. In Appendix~\ref{app:sc_prob}, we provide an additional ablation of $p$ values.

\paragraph{q-sampling}
We observe that q-sampling improves generation quality both with and without SC, and similarly to SC provides a better trade-off between perplexity and diversity.
A key practical advantage over SC is that q-sampling requires no retraining, whereas SC modifies the training protocol and increases training time.
The hyperparameter $t_{\mathrm{act}}$ can be tuned by the inference-time grid-search within a narrow range ($t_{\mathrm{act}} \in [0.3, 0.5]$).
In Appendix~\ref{app:q-sampling}, we show that the effect of varying $t_{\mathrm{act}}$ on text generation is consistent with the RHM experiments, and that its optimal value tends to be stable across datasets for a given model.
Moreover, Appendix~\ref{app:unconditional} shows the same behavior for OpenWebText and protein generation.
Best overall performance is achieved by combining q-sampling with SC.

\begin{tcolorbox}[colback=blue!5!white, colframe=blue!75!black]
\label{takeaway:3}
\textbf{Takeaway 3}: Self-conditioning is essential for unconditional generation, and tuning its probability $p$ above the default $p=0.5$ can yield significant further gains. Switching from DDPM to q-sampling also improves generation quality, and the best results are achieved by combining both techniques.
\end{tcolorbox}

%% file: chapters/conditional_generation.tex
\section{Conditional generation on real data}

% \begin{wrapfigure}{r}{0.67\textwidth}
%   \centering
%   \includegraphics[width=0.33\textwidth]{figures/bs_for_xsum.png}
%   \includegraphics[width=0.33\textwidth]{figures/cbs_for_conala.png}
%   \caption{XSum (left) and CoNaLa (right) quality for varying q-sampling activation timestep with and without SC.}
%   \label{fig:cond_quality}
%   \vspace{-1em}
% \end{wrapfigure}

\begin{wrapfigure}{r}{0.33\textwidth}
  \vspace{-2em}
  \includegraphics[width=0.33\textwidth]{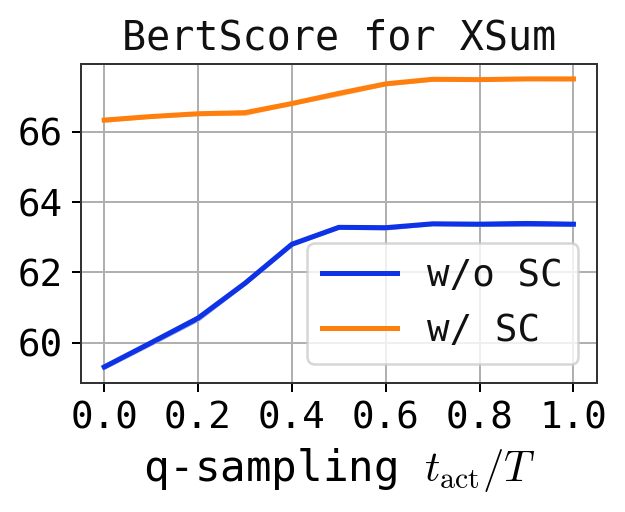}\\
  \includegraphics[width=0.33\textwidth]{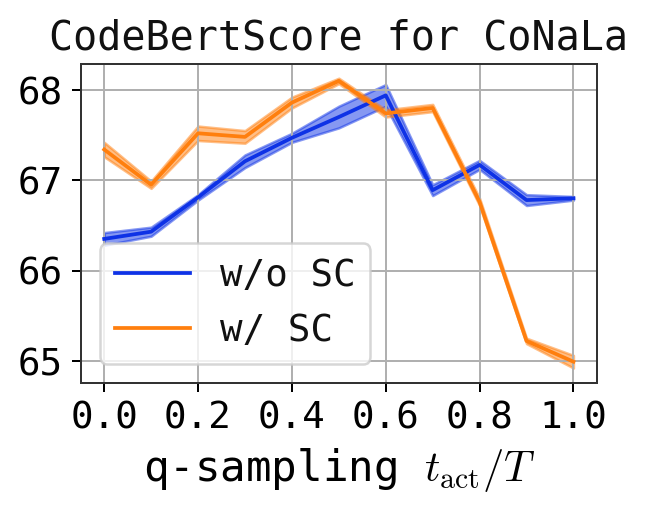}
  \caption{XSum (top) and CoNaLa (bottom) quality for varying q-sampling activation timestep with and without SC.}
  \label{fig:cond_quality}
  \vspace{-1em}
\end{wrapfigure}

We additionally evaluate SC and q-sampling on conditional generation, i.e. seq2seq tasks.
These differ fundamentally from unconditional generation, since the input sequence constrains the output.
We present results for TEncDM Emb trained on the XSum summarization task~\citep{xsum} and the GENIE model~\citep{genie} on the CoNaLa dataset~\citep{conala}, where the task is to generate a snippet of code given a natural language docstring.

Figure~\ref{fig:cond_quality} reports BertScore~\citep{bertscore} (XSum) and CodeBertScore~\citep{codebertscore} (CoNaLa) as a function of the q-sampling activation timestep, with and without SC.
Both techniques improve generation quality, but their contributions differ across tasks.
For XSum, SC yields a much larger gain, and the optimal q-sampling strategy is to apply it from the very beginning ($t_{\mathrm{act}}/T=1$).
For CoNaLa, SC provides a smaller improvement, and the optimal q-sampling activation timestep is $t_{\mathrm{act}}/T\approx0.5$.
We hypothesize that this discrepancy reflects the strength of the conditioning signal: the more tightly the input constrains the output, the more accurate the model predictions are even in the low-density regions. Therefore, there is less room for q-sampling and SC to improve generation.
In a sense, a weak conditioning signal (as in XSum and CoNaLa, where the correct output can vary substantially for a given input) makes seq2seq generation resemble unconditional generation, explaining why these techniques remain effective.
This intuition is further supported in Appendix~\ref{app:cond_generation} with results on machine translation and question paraphrasing, which have much stronger condition, and thus q-sampling and SC have little to no improvement.

\begin{tcolorbox}[colback=blue!5!white, colframe=blue!75!black]
\label{takeaway:4}
\textbf{Takeaway 4}: In seq2seq tasks, q-sampling and self-conditioning still improve generation quality, but their benefit is more significant for a weaker conditioning signal.
\end{tcolorbox}

%% file: chapters/conclusion.tex
\section{Conclusion}

In this work, we investigate why Gaussian diffusion with the DDPM solver underperforms on discrete data generation.
We find that DDPM trajectories tend to enter low-density regions during the critical interval --- the final phase of generation, where the distribution becomes multimodal.
Using the RHM toy example, we analyze how MBR, q-sampling, and SC mitigate this failure mode, and validate our findings on real data.
Combining SC and q-sampling substantially improves generation quality for shallow-embedding models, and to a lesser extent for context-dependent embedding models, suggesting that more continuous latent spaces are a promising direction for future research in text diffusion.

\section*{Acknowledgments}

The authors gratefully acknowledge the computing time made available to them on the high-performance computer at the NHR Center of TU Dresden. This center is jointly supported by the Federal Ministry of Research, Technology and Space of Germany and the state governments participating in the NHR (\url{www.nhr-verein.de/unsere-partner}).
We also thank Anton Ragin and QST Financial for providing additional computational resources.

%% file: chapters/appendix.tex
% \section*{Appendix Contents}
% \begin{description}
%   \item[A] \nameref{app:related_work} \dotfill \pageref{app:related_work}
%   \item[B] \nameref{app:rhm_description} \dotfill \pageref{app:rhm_description}
%   \item[C] \nameref{app:prediction_variance} \dotfill \pageref{app:prediction_variance}
%   \item[D] \nameref{app:implementation} \dotfill \pageref{app:implementation}
%   \item[E] \nameref{app:datasets} \dotfill \pageref{app:datasets}
%   \item[F] \nameref{app:mbr} \dotfill \pageref{app:mbr}
%   \item[G] \nameref{app:q-sampling} \dotfill \pageref{app:q-sampling}
%   \item[H] \nameref{app:number_of_steps} \dotfill \pageref{app:number_of_steps}
%   \item[I] \nameref{app:ideas} \dotfill \pageref{app:ideas}
%   \item[J] \nameref{app:unconditional} \dotfill \pageref{app:unconditional}
%   \item[K] \nameref{app:cond_generation} \dotfill \pageref{app:cond_generation}
%   \item[L] \nameref{app:image_diffusion} \dotfill \pageref{app:image_diffusion}
%   \item[M] \nameref{app:image_toy} \dotfill \pageref{app:image_toy}
% \end{description}

\input{chapters/appendix/related_work.tex}

\input{chapters/appendix/rhm}

\input{chapters/appendix/midpoint}

\input{chapters/appendix/q-sampling_theory}

\input{chapters/appendix/implementation}

\input{chapters/appendix/datasets}

\input{chapters/appendix/mbr_examples}

\input{chapters/appendix/q-sampling_ablations}

\input{chapters/appendix/number_of_steps}

\input{chapters/appendix/sc_prob}

\input{chapters/appendix/unsuccessful_ideas}

\input{chapters/appendix/unconditional}

\input{chapters/appendix/conditional}

\input{chapters/appendix/image_examples.tex}

%% file: chapters/appendix/related_work.tex
\section{Related work}\label{app:related_work}

While we conduct a thorough analysis of specifics of the continuous text diffusion's generation process, several previous works have attempted to identify the most critical diffusion components and to understand why text diffusions underperform their image counterparts.

\citet{difformer, dinoiser} and \citet{tencdm} suggest that text diffusion models require more noise at the initial steps of the forward process, because embeddings are well-separated in the latent space and a small amount of noise does not destroy any useful information, which makes the corresponding timesteps useless during generation. However, when embeddings are trained jointly with the diffusion model using Eq~\ref{eq:full_loss}, they begin to collapse until the decoder can no longer distinguish them, so that even a small amount of noise corrupts them \citep{diffusion-lm}. In this regime, more aggressive noise schedules become redundant.

\citet{cdcd} further proposes to select the noise schedule adaptively based on the principle that reconstruction loss should increase linearly with the timestep. \citet{seqdiffuseq} builds on it and suggests fitting the noise schedule for each token position independently.

\citet{tencdm} analyzes the impact of self-conditioning on generation and find that it increases the magnitude of the model's outputs, which can be interpreted as a measure of confidence. This property can cause a large mismatch between the self-conditioning inputs seen during training and those encountered at inference, since the magnitude at inference is much larger. Based on this observation, the authors propose reducing the number of denoising steps to limit this magnitude growth, which improves generation quality.

\citet{smoothie} proposes using q-sampling for a specific variant of Gaussian diffusion, arguing that it yields better quality than DDPM on seq2seq tasks, though without explaining the source of this improvement. \citet{tess} uses the same solver but rounds the model's predictions to the nearest embedding at each denoising step. This ensures that all predictions match the magnitude of the original data, at the cost of introducing bias into the predictions.

\citet{pynadath2025candi} provide a complementary perspective on the failures of continuous text diffusion. They consider two perspective of signal degradation: discrete identity corruption (is the correct token identifiable?) and continuous rank degradation (how distinguishable is the correct token?), attributing model failures to a temporal dissonance between the two.
However, their analysis is limited to one-hot token representations, which disregard semantic relationships between tokens.
As a result, their findings do not directly transfer to the more common setting considered in this paper, where tokens are represented as shallow or context-dependent embeddings and such dissonance does not arise.

%% file: chapters/appendix/rhm.tex
\section{Random Hierarchy Model: setup description}\label{app:rhm_description}

\paragraph{Data generation}
RHM consists of $L$ hierarchy levels.
At each level $l$, a feature from vocabulary $V$ is expanded into a tuple of $s$ lower-level features. Each feature has $m$ distinct but equivalent expansions (synonyms).
This creates a tree of depth $L$ and branching factor $s$, where $n = s^L$ leaves are the observable input features.
The number of possible sequences to generate is $N = m^{\sum_{i=0}^{L - 1}s^i} = m^{\frac{n - 1}{s - 1}}$, while the number of all unique sequences is $V^n$. In our setting, $L = 3,\,m = 4,\,s = 2,\,V = 16$, meaning that $N = 16384$ and $V^n \approx 4\cdot10^9$. The sequence length $n = 8$.
Throughout the paper, we refer to all sequences that can be generated by RHM as \emph{correct} and all others as \emph{incorrect}.
We use $25$ steps during sampling with \emph{sqrt} noise schedule~\citep{diffusion-lm}. Our diffusion model has a Transformer-encoder architecture with 4 Transformer layers and the hidden size of 64. Attention layers have 4 attention heads and the imtermediate hidden size in FFN is 64. To work with the inputs of dimention $d=2$, we upsample the input and downsample the output with linear layers.

\paragraph{Density and optimal model}
The density of noisified data $p(\textbf{x}_t)$ is given by
\begin{align}
    \label{eq:density}
    p(\mathbf{x}_t) = \sum_{\mathbf{x}_0} p(\mathbf{x}_t|\mathbf{x}_0) p(\mathbf{x}_0) = \frac{1}{(2 \pi (1 - \bar{\alpha}_t))^{\frac{nd}{2}}} \frac{1}{N} \sum_{\mathbf{x}_0} \exp\left(- \frac{\|\mathbf{x}_t - \sqrt{\bar{\alpha}_t} \mathbf{x}_0\|^2}{2(1 - \bar{\alpha}_t)} \right)
\end{align}
The optimal prediction $\mathbf{x}^*_0$ for $\mathbf{x}_t$ that minimizes Eq~\ref{eq:mse_loss} can be explicitly calculated as
\begin{align}
    \mathbf{x}^*_0(\mathbf{x}_t) = \mathbb{E}[\mathbf{x}_0 | \mathbf{x}_t] = \sum_{\mathbf{x}_0} \mathbf{x}_0 p(\mathbf{x}_0 | \mathbf{x}_t) = \frac{1}{N}\sum_{\mathbf{x}_0} \mathbf{x}_0 \frac{p(\mathbf{x}_t | \mathbf{x}_0)}{p(\mathbf{x}_t)}
\end{align}

\paragraph{Measuring one-mode rate}
To determine whether two points $\mathbf{x}^1_t$, $\mathbf{x}^2_t$ lie within the same probability mode, we calculate a log-density barrier $B(\mathbf{x}^1_t, \mathbf{x}^2_t)$ inspired by research on linear mode connectivity~\citep{loss_barrier}:
\[
B(\mathbf{x}^1_t, \mathbf{x}^2_t) = \inf_{\alpha \in [0, 1]} \Big\{\log p\big(\alpha \mathbf{x}^1_t + (1-\alpha) \mathbf{x}^2_t\big) - \alpha \log p (\mathbf{x}^1_t) - (1 - \alpha) \log p(\mathbf{x}^2_t) \Big\}
\]
When $B(\mathbf{x}^1_t, \mathbf{x}^2_t) < 0$, then there exists a point on a line segment between $\mathbf{x}^1_t$ and $\mathbf{x}^2_t$, where the log-density is below the interpolated log-density at the endpoints, implying that the endpoints are located in different modes.

%% file: chapters/appendix/midpoint.tex
\section{Low-density latent samples result in incorrect predictions}\label{app:midpoint_prediction}

\begin{figure}
\centering
\includegraphics[width=0.7\textwidth]{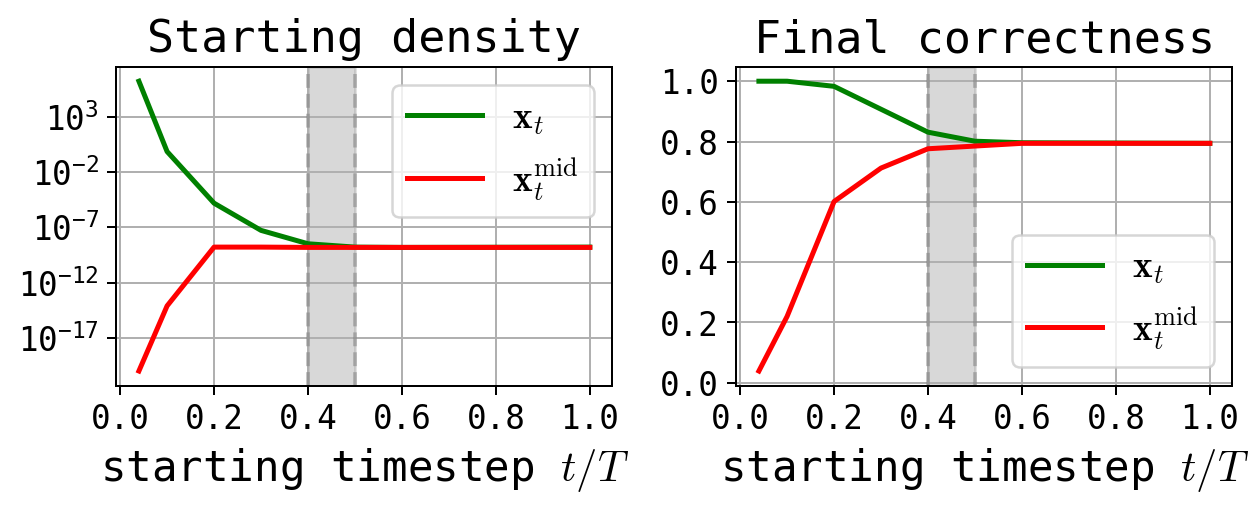}
\caption{(left) Density of the starting latent sample. (right) Correctness of the generation trajectories started from $\mathbf{x}_t$ and $\mathbf{x}^{\text{mid}}_t$. The \emph{transition interval} defined in Section \ref{sec:exploratory} is marked with the \textcolor{gray}{gray} band.}
\label{fig:midpoint}
\end{figure}

In this section, we empirically support the claim that latents $\mathbf{x}_t$ with low density tend to be denoised into incorrect samples. On the toy RHM example, we compare trajectories started from $\mathbf{x}_t \sim q(\mathbf{x}_t \mid \mathbf{x}_0)$ with trajectories started from a deliberately low-density latent. Specifically, we take a norm-preserving average of two independent latents, $\mathbf{x}^{\text{mid}}_t = (\mathbf{x}^{(1)}_t + \mathbf{x}^{(2)}_t)/\sqrt{2}$, where $\mathbf{x}^{(i)}_t \sim q(\mathbf{x}_t \mid \mathbf{x}^{(i)}_0)$. The rationale is that after the \emph{transition interval}, $\mathbf{x}^{(1)}_t$ and $\mathbf{x}^{(2)}_t$ lie in different modes with high probability, so their midpoint falls between the modes and has low density. This construction isolates the effect of the starting density and lets us test directly whether the model can recover once a trajectory enters a low-density region.

Figure~\ref{fig:midpoint} (left) plots the density of the two starting points as a function of the starting timestep $t/T$. Before the transition interval, where the latent distribution is still unimodal, the two densities coincide: averaging two samples from a single mode yields another typical sample. After it, once the modes separate, the density of $\mathbf{x}^{\text{mid}}_t$ drops by many orders of magnitude, making it an out-of-distribution input for the model. Figure~\ref{fig:midpoint} (right) shows the correctness of the corresponding generations and mirrors this picture. Before the transition interval the two starting points are indistinguishable in terms of the final correctness, whereas after it the curves separate: correctness of the $\mathbf{x}_t$ trajectories approaches $1$, while correctness of the $\mathbf{x}^{\text{mid}}_t$ trajectories collapses toward $0$. The onset of this drop coincides with the onset of the density gap, confirming that the diffusion model rarely recovers once its trajectory enters a low-density region.

%% file: chapters/appendix/q-sampling_theory.tex
\section{Variance of latents updated with q-sampling and DDPM}\label{app:prediction_variance}

In this appendix, we show that q-sampling deviates more from the forward process than DDPM: although both solvers produce marginal variances below that of the forward process, the marginal variance under q-sampling is even lower.
To simplify the analysis, we assume that $\bx_t$ is sampled from the marginal density $p(\bx_t)$ and compare the variances of the one-step updates $\bx_{t-1}^{\mathrm{DDPM}}$ and $\bx_{t-1}^{\mathrm{QS}}$.
 
\subsection{Setup}
 
Consider a data distribution supported on a finite set of points,
$p_{\mathrm{data}}(\bx) = \frac{1}{N} \sum_{\bx_0} \delta(\bx - \bx_0)$,
and a forward diffusion process
\begin{equation}
    \bx_t = \sqrt{\bar\alpha_t}\,\bx_0 + \sqrt{1 - \bar\alpha_t}\,\be, \qquad \be \sim \mathcal{N}(0, \bI).
\end{equation}
 
We compare two reverse sampling procedures, both using the Bayes-optimal prediction $\mathbf{x}^*_0(\bx_t) := \E[\bx_0 | \bx_t]$, with $\bx_t \sim p(\bx_t)$.
 
\paragraph{DDPM solver}
\begin{equation}\label{eq:ddpm}
    \bx_{t-1}^{\mathrm{DDPM}} = a\,\mathbf{x}^*_0(\bx_t) + b\,\bx_t + \sqrt{\tilde\beta_t}\,\be', \qquad \be' \sim \mathcal{N}(0, \bI),
\end{equation}
where
\begin{equation}
    a = \frac{\sqrt{\bar\alpha_{t-1}}\,\beta_t}{1 - \bar\alpha_t}, \qquad
    b = \frac{\sqrt{\alpha_t}\,(1 - \bar\alpha_{t-1})}{1 - \bar\alpha_t}, \qquad
    \tilde\beta_t = \frac{1 - \bar\alpha_{t-1}}{1 - \bar\alpha_t}\,\beta_t.
\end{equation}
 
\paragraph{q-sampling solver}
\begin{equation}\label{eq:qsampling}
    \bx_{t-1}^{\mathrm{QS}} = \sqrt{\bar\alpha_{t-1}}\,\mathbf{x}^*_0(\bx_t) + \sqrt{1 - \bar\alpha_{t-1}}\,\be'', \qquad \be'' \sim \mathcal{N}(0, \bI).
\end{equation}
 
In both cases, $\be'$ and $\be''$ are independent of $(\bx_0, \bx_t)$.
 
\subsection{Result}
 
\begin{theorem}\label{prop:main}
Let $\bx_t \sim p(\bx_t)$ and define $\bV_t := \E[\Var[\bx_0 | \bx_t]]$, the average posterior uncertainty. Then\footnote{We operate with covariance matrices, so $\mathbf{A}\ge \mathbf{B}$ denotes that $\mathbf{A}-\mathbf{B}$ is positive semidefinite.}
\begin{equation}
    \Var[\bx_{t-1}^{\mathrm{DDPM}}] - \Var[\bx_{t-1}^{\mathrm{QS}}] = b^2 \bar\alpha_t \, \bV_t \geq 0,
\end{equation}

Moreover, both underestimate the true forward marginal variance:
\begin{align}
    \Var_p[\bx_{t-1}] - \Var[\bx_{t-1}^{\mathrm{DDPM}}] &= (\bar\alpha_{t-1} - b^2\bar\alpha_t)\,\bV_t, \\
    \Var_p[\bx_{t-1}] - \Var[\bx_{t-1}^{\mathrm{QS}}] &= \bar\alpha_{t-1}\,\bV_t.
\end{align}
\end{theorem}
 
\begin{proof}
Since $\bx_t \sim p(\bx_t)$, there exists an underlying $\bx_0 \sim p_{\mathrm{data}}(\bx_0)$ such that $\bx_t = \sqrt{\bar\alpha_t}\,\bx_0 + \sqrt{1-\bar\alpha_t}\,\be$. We decompose $\bx_0$ into its prediction and residual:
\begin{equation}\label{eq:decomp}
    \bx_0 = \E[\bx_0|\bx_t] + \bd, \qquad \bd := \bx_0 - \E[\bx_0|\bx_t].
\end{equation}
By the defining property of conditional expectation, $\bd$ is uncorrelated with any function of $\bx_t$, and in particular with $\E[\bx_0|\bx_t]$. Moreover,
\begin{equation}\label{eq:Vt}
    \E[\bd\bd^\top] = \E[\Var[\bx_0|\bx_t]] =: \bV_t.
\end{equation}
 
\medskip\noindent\textbf{Step 1: Variance of the q-sampling solver.}
Since $\E[\bx_0|\bx_t]$ is a function of $\bx_t$ and $\be''$ is independent:
\begin{equation}\label{eq:var_q}
    \Var[\bx_{t-1}^{\mathrm{QS}}]
    = \bar\alpha_{t-1}\,\Var[\E[\bx_0|\bx_t]] + (1 - \bar\alpha_{t-1})\,\bI.
\end{equation}
 
\medskip\noindent\textbf{Step 2: Expanding the DDPM mean.}
Substituting the decomposition from Eq.~\ref{eq:decomp} into $\bx_t = \sqrt{\bar\alpha_t}\,\bx_0 + \sqrt{1-\bar\alpha_t}\,\be$:
\begin{equation}
    \bx_t = \sqrt{\bar\alpha_t}\,\E[\bx_0|\bx_t] + \sqrt{\bar\alpha_t}\,\bd + \sqrt{1-\bar\alpha_t}\,\be.
\end{equation}
The DDPM mean $\mu(\bx_t) := a\,\E[\bx_0|\bx_t] + b\,\bx_t$ becomes:
\begin{equation}\label{eq:m_expand}
    \mu(\bx_t) = \underbrace{(a + b\sqrt{\bar\alpha_t})}_{\sqrt{\bar\alpha_{t-1}}}\,\E[\bx_0|\bx_t] + b\sqrt{\bar\alpha_t}\,\bd + b\sqrt{1-\bar\alpha_t}\,\be,
\end{equation}
where the identity $a + b\sqrt{\bar\alpha_t} = \sqrt{\bar\alpha_{t-1}}$ follows from:
\begin{align}
    a + b\sqrt{\bar\alpha_t}
    &= \frac{\sqrt{\bar\alpha_{t-1}}\,\beta_t}{1-\bar\alpha_t}
     + \frac{\sqrt{\alpha_t}(1-\bar\alpha_{t-1})}{1-\bar\alpha_t}\sqrt{\bar\alpha_t} = \gray{\Big\{\text{since }\sqrt{\alpha_t}\sqrt{\bar{\alpha}_t}=\alpha_t\sqrt{\bar{\alpha}_{t-1}}\Big\}}\nonumber\\
    &= \frac{\sqrt{\bar\alpha_{t-1}}\big[\beta_t + \alpha_t(1-\bar\alpha_{t-1})\big]}{1-\bar\alpha_t}
    = \frac{\sqrt{\bar\alpha_{t-1}}(1-\bar\alpha_t)}{1-\bar\alpha_t}
    = \sqrt{\bar\alpha_{t-1}},
\end{align}
using $\beta_t + \alpha_t(1-\bar\alpha_{t-1}) = (1-\alpha_t) + \alpha_t - \alpha_t\bar\alpha_{t-1} = 1 - \bar\alpha_t$.
 
\medskip\noindent\textbf{Step 3: Variance of the DDPM solver.}
The three terms in Eq.~\ref{eq:m_expand} are mutually uncorrelated: $\E[\bx_0|\bx_t]$ is a function of $\bx_t$, $\bd$ is uncorrelated with any function of $\bx_t$ by construction (Eq.~\ref{eq:decomp}), and $\be$ is independent of $\bx_0$. Therefore:
\begin{equation}
    \Var[\mu(\bx_t)] = \bar\alpha_{t-1}\,\Var[\E[\bx_0|\bx_t]]
    + b^2\bar\alpha_t\,\bV_t
    + b^2(1-\bar\alpha_t)\,\bI.
\end{equation}
Adding the noise term $\tilde\beta_t \be'$:
\begin{equation}\label{eq:var_ddpm}
    \Var[\bx_{t-1}^{\mathrm{DDPM}}]
    = \bar\alpha_{t-1}\,\Var[\E[\bx_0|\bx_t]]
    + b^2\bar\alpha_t\,\bV_t
    + \underbrace{\big(b^2(1-\bar\alpha_t) + \tilde\beta_t\big)}_{1-\bar\alpha_{t-1}}\,\bI,
\end{equation}
where the identity $b^2(1-\bar\alpha_t) + \tilde\beta_t = 1-\bar\alpha_{t-1}$ follows from:
\begin{align}
    b^2(1-\bar\alpha_t) + \tilde\beta_t
    &= \frac{\alpha_t(1-\bar\alpha_{t-1})^2}{(1-\bar\alpha_t)^2}(1-\bar\alpha_t)
     + \frac{(1-\bar\alpha_{t-1})\beta_t}{1-\bar\alpha_t} \nonumber\\
    &= \frac{(1-\bar\alpha_{t-1})\big[\alpha_t(1-\bar\alpha_{t-1}) + \beta_t\big]}{1-\bar\alpha_t}
    = \frac{(1-\bar\alpha_{t-1})(1-\bar\alpha_t)}{1-\bar\alpha_t}
    = 1-\bar\alpha_{t-1}.
\end{align}
 
\medskip\noindent\textbf{Step 4: Comparison.}
Subtracting Eq.~\ref{eq:var_q} from Eq.~\ref{eq:var_ddpm}:
\begin{equation}
    \boxed{\Var[\bx_{t-1}^{\mathrm{DDPM}}] - \Var[\bx_{t-1}^{\mathrm{QS}}] = b^2\bar\alpha_t\,\bV_t \geq 0.}
\end{equation}
 
\medskip\noindent\textbf{Step 5: Comparison with the true marginal.}
The forward marginal variance is $\Var_p[\bx_{t-1}] = \bar\alpha_{t-1}\,\Var[\bx_0] + (1-\bar\alpha_{t-1})\bI$. By the law of total variance,
\begin{equation}
    \Var[\bx_0] = \Var[\E[\bx_0|\bx_t]] + \bV_t,
\end{equation}
so $\Var[\E[\bx_0|\bx_t]] = \Var[\bx_0] - \bV_t$. Substituting into Eq.~\ref{eq:var_q} and Eq.~\ref{eq:var_ddpm}:
\begin{align}
    \Var_p[\bx_{t-1}] - \Var[\bx_{t-1}^{\mathrm{QS}}] &= \bar\alpha_{t-1}\,\bV_t, \\
    \Var_p[\bx_{t-1}] - \Var[\bx_{t-1}^{\mathrm{DDPM}}] &= (\bar\alpha_{t-1} - b^2\bar\alpha_t)\,\bV_t.
\end{align}
\end{proof}
 
\begin{remark}
The difference $b^2\bar\alpha_t\,\bV_t$ vanishes in two regimes: (i) when $\bV_t \to 0$, i.e., the posterior $p(\bx_0|\bx_t)$ concentrates on a single point (small $t$); (ii) when $\bar\alpha_t \to 0$ (large $t$, near pure noise). The gap is largest at intermediate noise levels where the posterior is genuinely uncertain yet $\bar\alpha_t$ is still appreciable.
\end{remark}
 
\begin{remark}
Both solvers underestimate the true marginal variance, but by different amounts. The $q$-sampling solver discards all posterior uncertainty through resampling fresh noise, losing $\bar\alpha_{t-1}\bV_t$. The DDPM solver partially recovers this through the coupling between $\E[\bx_0|\bx_t]$ and $\bx_t$ in its mean formula, retaining $b^2\bar\alpha_t \bV_t$ and losing only $(\bar\alpha_{t-1} - b^2\bar\alpha_t)\bV_t$.
\end{remark}

%% file: chapters/appendix/implementation.tex
\section{Implementation details}\label{app:implementation}

In this paper, we train models from prior work using their original code, keeping the tokenization methods, model architectures and all hyperparameters unchanged. Below we describe all introdused modifications which were required for our analyzis.

\paragraph{Self-conditioning}
Diffusion-LM~\citep{diffusion-lm} and GENIE~\citep{genie} do not include SC in their original implementations.
We add SC support following the approach of SeqDiffuSeq~\citep{seqdiffuseq}, since all three models share the same codebase.
In all our experiments, we the same model for generating with and without SC.
As every model trained with SC also learns to make predictions without it, we can turn SC on and off during the inference to avoid model retrainig.

\paragraph{Adaptation of GENIE to code generation}
The original GENIE paper does not provide a code generation pipeline, and while CodeFusion~\citep{codefusion} builds on the GENIE method, no implementation is publicly available.
We therefore reimplement the CodeFusion setup using GENIE's source code.
Additionally, the original approach includes a pre-training stage before fine-tuning; we omit this step to reduce computational cost, as the original paper reports only a marginal impact on quality, and we focus on the generation process rather than achieving state-of-the-art performance.

\paragraph{Number of generation steps}
Diffusion-LM, SeqDiffuSeq, and GENIE originally use more than $1000$ generation steps, which we find excessive and impractical.
We reduce the number of steps to $200$ for Diffusion-LM, $100$ for GENIE, and $50$ for SeqDiffuSeq; preliminary experiments confirm that this has a negligible impact on seq2seq quality.
For TEncDM-based models~\citep{tencdm}, we use $50$ steps as recommended in the original paper, except for MBR on ROCStories and OpenWebText, where we use $100$ steps.
The increased step count for MBR is chosen to better demonstrate the effect, as $50$ steps yield lower quality for embedding-based diffusion and make the improvement less pronounced.
For OpenWebText, which contains sequences of $512$ tokens (substantially longer than other datasets, see Table~\ref{tab:seq_lengths}), $100$ steps also better suit the task.
The effect of the number of generation steps on unconditional generation quality is analyzed in Appendix~\ref{app:number_of_steps}.

\begin{table}
\centering
\begin{tabular}{lcccc}
\toprule
 & \textbf{RHM} & \textbf{ROCStories} & \textbf{Wikipedia} & \textbf{OWT} \\
\midrule
Max length & $8$ & $80$ & $128$ & $512$ \\
\midrule
\midrule
& \textbf{XSum} & \textbf{QQP} & \textbf{IWSLT14} & \textbf{CoNaLa} \\
\midrule
Max input length & $512$ & $50$ & $64$ & $64$ \\
Max target length & $64$ & $50$ & $64$ & $64$ \\
\bottomrule
\end{tabular}
\caption{Max sequence lengths used for each dataset.}
\label{tab:seq_lengths}
\end{table}

%% file: chapters/appendix/datasets.tex
\section{Datasets}\label{app:datasets}

\subsection{Unconditional generation}

\textbf{ROCStories}\;
The dataset~\citep{rocstories} contains five-sentence commonsense stories and consists of $98\,161$ instances. $93\,161$ instances are held out for training, and $5\,000$ instances are used for validation.

\textbf{Wikipedia}\;
For experiments of larger scale, we utilize the English Wikipedia dataset~\citep{wikidump}.
The dataset comprises a total of $6.4$M sequences, of which $10$k are allocated as a validation set.

\textbf{OpenWebText}\;
OpenWebText \citep{openwebtext} is the largest dataset we use. It contains 8M texts, which we truncate to the length of $512$ tokens without any additional preprocessing. We allocate $100$k samples for validation and keep the remaining texts for training.

\subsection{Conditional generation}

\textbf{XSum}\;
The dataset~\citep{xsum} is used for summarization task and it contains $204$k BBC articles, which are provided as document and summary pairs and covered wide range of topics (Sports, Politics, etc.).
It has $204\,045$ training instances, $11\,332$ validation instances, and $11\,334$ test instances.

\textbf{QQP}\;
The subset of QQP dataset, proposed in~\citet{qqp}, consists of $149$k question pairs from the Quora platform that are paraphrases of each other.
The data is splitted in $144\,715$ training instances, $2\,048$ validation instances, and $2\,500$ test instances.

\textbf{IWSLT14}\;
IWSLT14 is a widely used benchmark for machine translation. We use the German-English language pairs and translate from German to English. We follow the preprocessing pipeline of fairseq \citep{fairseq}, which relies on the Moses tokenizer \citep{moses}.
In IWSLT14, $160\,239$ pairs are allocated for training, $7\,283$ are used for validation, and $6\,750$ for testing.

\textbf{CoNaLa}\;
The CoNaLa dataset \citep{conala} consists of complex, multi-statement StackOverflow Python code snippets and associated NL questions. It has $594$k train pairs and $500$ test pairs.

%% file: chapters/appendix/mbr_examples.tex
\section{Minimum Bayes-Risk examples}\label{app:mbr}

In this section, we present examples of MBR selection for RHM and ROCStories dataset. 

\subsection{Random Hierarchy Model}\label{app:mbr_rhm}

MBR for two different generation trajectories is presented in Table~\ref{tab:mbr_rhm_example}.
For the smallest activation timesteps $t_{\mathrm{act}}/T = 0.1$ and $t_{\mathrm{act}}/T = 0.3$, the candidates are either identical or differ in only a few tokens.
MBR is therefore ineffective in this regime: the sample has been mostly determined (i.e., the prediction $\hat{\mathbf{x}}_{\theta}$ at $t_{\mathrm{act}}/T = 0.1$ coincides with the majority of candidates), leaving no room to correct invalid outputs.
At $t_{\mathrm{act}}/T = 0.5$, by contrast, the candidate sequences are substantially more diverse, containing both correct and incorrect samples; here MBR successfully identifies and selects the correct ones.
Finally, at $t_{\mathrm{act}}/T = 0.7$ the candidates are too diverse for averaging to be meaningful: in the presented examples, MBR selects an incorrect sequence even though correct candidates are present in the set.

\begin{table}[]
\centering
\begin{tabular}{l|cccc}
\toprule
 & $t_{\mathrm{act}}/T=0.1$ & $t_{\mathrm{act}}/T=0.3$ & $t_{\mathrm{act}}/T=0.5$ & $t_{\mathrm{act}}/T=0.7$ \\
\midrule
\midrule
MBR candidate &
\textcolor{red!50!black}{\texttt{b\,l\,b\,b\,j\,b\,j\,k
}} &
\textcolor{green!50!black}{\texttt{b\,l\,f\,b\,j\,b\,j\,k
}}&
\textcolor{green!50!black}{\texttt{i\,n\,f\,b\,l\,d\,d\,i
}} &
\textcolor{red!50!black}{\texttt{k\,n\,j\,i\,o\,a\,p\,a
}} \\
\midrule
Candidate \#1 &
\textcolor{red!50!black}{\texttt{b\,l\,b\,b\,j\,b\,j\,k
}} &
\textcolor{green!50!black}{\texttt{i\,c\,b\,b\,j\,b\,j\,k
}} &
\textcolor{green!50!black}{\texttt{i\,n\,f\,b\,l\,d\,d\,i
}} &
\textcolor{red!50!black}{\texttt{k\,g\,o\,c\,k\,i\,h\,n
}} \\

Candidate \#2 &
\textcolor{red!50!black}{\texttt{b\,l\,b\,b\,j\,b\,j\,k
}} &
\textcolor{green!50!black}{\texttt{i\,n\,f\,b\,j\,b\,j\,k
}} &
\textcolor{red!50!black}{\texttt{n\,o\,b\,h\,k\,o\,k\,i
}} &
\textcolor{green!50!black}{\texttt{h\,n\,c\,e\,f\,b\,l\,e
}} \\

Candidate \#3 &
\textcolor{red!50!black}{\texttt{b\,l\,b\,b\,j\,b\,j\,k
}} &
\textcolor{red!50!black}{\texttt{b\,l\,f\,b\,j\,b\,o\,k
}} &
\textcolor{green!50!black}{\texttt{e\,n\,f\,b\,o\,o\,m\,b
}} &
\textcolor{red!50!black}{\texttt{k\,n\,j\,i\,o\,a\,p\,a
}} \\

Candidate \#4 &
\textcolor{red!50!black}{\texttt{b\,l\,b\,b\,j\,b\,j\,k\
}} &
\textcolor{green!50!black}{\texttt{b\,l\,f\,b\,j\,b\,j\,k
}} &
\textcolor{green!50!black}{\texttt{i\,c\,b\,b\,l\,d\,j\,b
}} &
\textcolor{red!50!black}{\texttt{j\,k\,j\,m\,a\,i\,o\,a
}} \\

Candidate \#5 &
\textcolor{red!50!black}{\texttt{b\,l\,b\,b\,j\,b\,j\,k
}} &
\textcolor{green!50!black}{\texttt{b\,l\,f\,b\,j\,b\,j\,k
}} &
\textcolor{red!50!black}{\texttt{b\,g\,f\,b\,j\,b\,h\,b
}} &
\textcolor{green!50!black}{\texttt{d\,e\,j\,d\,p\,b\,c\,b
}} \\
\midrule
Prediction $\hat{\mathbf{x}}_{\theta}$ at $t_{\mathrm{act}}$ &
\textcolor{red!50!black}{\texttt{b\,l\,b\,b\,j\,b\,j\,k
}} &
\textcolor{red!50!black}{\texttt{b\,l\,f\,b\,j\,b\,m\,k
}} &
\textcolor{red!50!black}{\texttt{n\,m\,n\,b\,m\,i\,m\,b
}} &
\textcolor{red!50!black}{\texttt{n\,p\,h\,n\,h\,n\,m\,n
}} \\

\midrule
\midrule

MBR candidate &
\textcolor{green!50!black}{\texttt{l\,d\,m\,m\,k\,o\,d\,a
}}&
\textcolor{green!50!black}{\texttt{l\,d\,m\,m\,k\,o\,d\,a
}} &
\textcolor{green!50!black}{\texttt{j\,d\,o\,c\,k\,i\,d\,a
}} &
\textcolor{red!50!black}{\texttt{j\,d\,o\,c\,o\,a\,k\,o
}} \\

\midrule

Candidate \#1 &
\textcolor{green!50!black}{\texttt{l\,d\,m\,m\,k\,o\,d\,a
}} &
\textcolor{green!50!black}{\texttt{l\,d\,m\,m\,k\,o\,d\,a
}} &
\textcolor{green!50!black}{\texttt{j\,d\,o\,c\,k\,o\,d\,a
}} &
\textcolor{red!50!black}{\texttt{j\,d\,o\,c\,o\,a\,k\,o
}} \\

Candidate \#2 &
\textcolor{green!50!black}{\texttt{l\,d\,m\,m\,k\,o\,d\,a
}} &
\textcolor{green!50!black}{\texttt{l\,d\,m\,m\,k\,i\,d\,a
}} &
\textcolor{green!50!black}{\texttt{p\,h\,o\,c\,k\,i\,d\,i
}} &
\textcolor{green!50!black}{\texttt{j\,d\,o\,c\,o\,k\,o\,f
}} \\

Candidate \#3 &
\textcolor{green!50!black}{\texttt{l\,d\,m\,m\,k\,o\,d\,a
}} &
\textcolor{green!50!black}{\texttt{l\,d\,m\,m\,k\,o\,d\,a
}} &
\textcolor{green!50!black}{\texttt{j\,d\,o\,c\,k\,i\,d\,a
}} &
\textcolor{green!50!black}{\texttt{m\,m\,o\,f\,e\,c\,d\,i
}}\\

Candidate \#4 &
\texttt{\textcolor{green!50!black}{l\,d\,m\,m\,k\,i\,d\,a
}} &
\textcolor{green!50!black}{\texttt{l\,d\,m\,m\,k\,o\,d\,a
}} &
\textcolor{green!50!black}{\texttt{j\,d\,o\,c\,k\,i\,d\,i
}} &
\textcolor{green!50!black}{\texttt{n\,o\,i\,l\,i\,c\,l\,e
}} \\

Candidate \#5 &
\textcolor{green!50!black}{\texttt{l\,d\,m\,m\,k\,o\,d\,a
}} &
\textcolor{green!50!black}{\texttt{l\,d\,m\,m\,k\,i\,d\,a
}} &
\textcolor{green!50!black}{\texttt{j\,d\,o\,c\,k\,o\,d\,a
}} &
\textcolor{red!50!black}{\texttt{j\,k\,b\,l\,i\,c\,b\,a
}} \\

\midrule

Prediction $\hat{\mathbf{x}}_{\theta}$ at $t_{\mathrm{act}}$ &
\textcolor{green!50!black}{\texttt{l\,d\,m\,m\,k\,o\,d\,a
}} &
\textcolor{green!50!black}{\texttt{l\,d\,m\,m\,k\,i\,d\,a
}} &
\textcolor{red!50!black}{\texttt{j\,d\,o\,c\,d\,i\,d\,f
}} &
\textcolor{red!50!black}{\texttt{m\,o\,o\,m\,h\,o\,o\,f
}} \\

\bottomrule
\end{tabular}
\caption{Five MBR candidates and chosen samples for different activation timesteps $t_{\mathrm{act}}$. We use latin letters to represent each token. Correct sequences are highlighted in \textcolor{green!50!black}{\texttt{green}} and incorrect in \textcolor{red!50!black}{\texttt{red}}.
Activating MBR at $t_\mathrm{act}/T = 0.3$ or $t_\mathrm{act}/T = 0.5$ ensures that the correct sequence is sampled; this would not be the case without MBR.}
\label{tab:mbr_rhm_example}
\end{table}

\subsection{ROCStories}\label{app:mbr_roc}

Table~\ref{tab:mbr_text_example} presents a representative set of examples for the TEncDM Emb model without self-conditioning on the ROCStories dataset.
At $t_{\mathrm{act}} = 0.3$, candidates differ only in a few words, typically proper names or contextual synonyms, reflecting that most of the generation has already been committed.
At $t_{\mathrm{act}} = 0.5$, variation extends to word combinations and phrasing, with candidates differing in stylistic suitability while preserving the overall theme.
At $t_{\mathrm{act}} = 0.7$, the texts diverge substantially, sharing only high-level structural elements such as key pronouns (e.g., ``they''), which leaves MBR with too little signal to reliably select the best candidate.
These examples confirm that MBR is most effective when applied around $t_{\mathrm{act}} = 0.5$, where candidates are similar enough for meaningful comparison yet diverse enough to benefit from selection.

\begin{table}[]
\centering
\addtolength{\tabcolsep}{-0.37em}
\begin{tabular}{l|c}
\toprule
\multicolumn{2}{c}{$t_{\mathrm{act}}/T=0.3$} \\
\midrule
\raisebox{-0.75\normalbaselineskip}[0pt][0pt]{\rotatebox{90}{MBR}} & \parbox{.95\linewidth}{\tinycolorbox{Green!25}{John} and John decided to build a lemonade stand. They \tinycolorbox{Green!25}{got} together in \tinycolorbox{Green!25}{a} big lemonade stand in their yard. picking out the lemonade, they put it in a \tinycolorbox{Green!25}{jar}. They ate the lemonade for a week. The neighbor told them to go home without their lemonade.} \\
\midrule
\#1 &
\parbox{.95\linewidth}{\tinycolorbox{red!25}{Amy} and John decided to build a lemonade stand. They \tinycolorbox{red!25}{walked} together in \tinycolorbox{Green!25}{a} big lemonade stand in their yard. picking out the lemonade, they put it in a \tinycolorbox{red!25}{bowl}. They ate the lemonade for a week. The neighbor told them to go home without their lemonade.} \\
\#2 &
\parbox{.95\linewidth}{\tinycolorbox{red!25}{Sam} and John decided to make a lemonade stand. They \tinycolorbox{Green!25}{got} together in \tinycolorbox{red!25}{the} big lemonade stand in their yard. picking out the lemonade, they put it in a \tinycolorbox{red!25}{bowl}. They ate the lemonade for a week. The neighbor told them to go home without their lemonade.} \\
\#3 &
\parbox{.95\linewidth}{\tinycolorbox{Green!25}{John} and John decided to build a lemonade stand. They \tinycolorbox{red!25}{sat} together in \tinycolorbox{Green!25}{a} large lemonade stand in their yard. picking out the lemonade, they put it in a \tinycolorbox{red!25}{box}. They ate the lemonade for a week. The neighbor told them to go home without their lemonade.} \\
\#4 &
\parbox{.95\linewidth}{\tinycolorbox{Green!25}{John} and John decided to make a lemonade stand. They \tinycolorbox{Green!25}{got} together in \tinycolorbox{Green!25}{a} large lemonade stand in their yard. picking out the lemonade, they put it in a \tinycolorbox{Green!25}{jar}. They ate the lemonade for a week. The neighbor told them to go home without their lemonade.} \\
\midrule
\multicolumn{2}{c}{$t_{\mathrm{act}}/T=0.5$} \\
\midrule
\raisebox{-0.75\normalbaselineskip}[0pt][0pt]{\rotatebox{90}{MBR}} & \parbox{.95\linewidth}{\tinycolorbox{Green!25}{Joe and Joe} decided to \tinycolorbox{Green!25}{have} a lemonade stand. They \tinycolorbox{Green!25}{quickly set up} \tinycolorbox{Green!25}{the best} lemonade stand in their yard. They \tinycolorbox{Green!25}{made} the lemonade, and \tinycolorbox{Green!25}{put it} in a \tinycolorbox{Green!25}{bowl}. They ate the lemonade for an hour. The \tinycolorbox{Green!25}{owner} told them to go \tinycolorbox{Green!25}{home} \tinycolorbox{Green!25}{with} the lemonade.} \\
\midrule
\#1 &
\parbox{.95\linewidth}{\tinycolorbox{red!25}{Larry and Bob} decided to \tinycolorbox{red!25}{make} a lemonade stand. They \tinycolorbox{red!25}{went down to} \tinycolorbox{Green!25}{the best} lemonade stand in their yard. They \tinycolorbox{red!25}{bought} the lemonade, and \tinycolorbox{Green!25}{put it} in a \tinycolorbox{red!25}{box}. They ate the lemonade for an hour. The \tinycolorbox{red!25}{neighbor} told them to go \tinycolorbox{red!25}{back} \tinycolorbox{Green!25}{with} the lemonade.} \\
\#2 &
\parbox{.95\linewidth}{\tinycolorbox{red!25}{Ron and Mark} decided to \tinycolorbox{red!25}{make} a lemonade stand. They \tinycolorbox{red!25}{found some of} \tinycolorbox{red!25}{the} lemonade stand \tinycolorbox{red!25}{stand} in their yard. They \tinycolorbox{red!25}{picked} the lemonade, and \tinycolorbox{red!25}{put them} in a \tinycolorbox{Green!25}{bowl}. They ate the lemonade for an hour. The \tinycolorbox{Green!25}{owner} told them to go \tinycolorbox{Green!25}{home} \tinycolorbox{red!25}{without} their lemonade.} \\
\#3 &
\parbox{.95\linewidth}{\tinycolorbox{Green!25}{Joe and Joe} decided to \tinycolorbox{Green!25}{have} a lemonade stand. They \tinycolorbox{Green!25}{quickly set up} \tinycolorbox{Green!25}{the best} lemonade stand in their yard. They \tinycolorbox{Green!25}{made} the lemonade, and \tinycolorbox{Green!25}{put it} in a \tinycolorbox{Green!25}{bowl}. They ate the lemonade for an hour. The \tinycolorbox{Green!25}{owner} told them to go \tinycolorbox{Green!25}{home} \tinycolorbox{Green!25}{with} the lemonade.} \\
\#4 &
\parbox{.95\linewidth}{\tinycolorbox{red!25}{Kim and Mark} decided to \tinycolorbox{Green!25}{have} a lemonade sale. They \tinycolorbox{red!25}{both went to} \tinycolorbox{red!25}{an ice} lemonade stand in the yard. They \tinycolorbox{red!25}{got} the lemonade, and \tinycolorbox{red!25}{placed them} in a \tinycolorbox{Green!25}{bowl}. They ate the lemonade for an hour. The \tinycolorbox{red!25}{manager} told them to go \tinycolorbox{Green!25}{home} \tinycolorbox{Green!25}{with} their lemonade.} \\
\midrule
\multicolumn{2}{c}{$t_{\mathrm{act}}/T=0.7$} \\
\midrule
\raisebox{-0.75\normalbaselineskip}[0pt][0pt]{\rotatebox{90}{MBR}} & \parbox{.95\linewidth}{\tinycolorbox{Green!25}{Ava and her} \tinycolorbox{Green!25}{friends} \tinycolorbox{Green!25}{went to} \tinycolorbox{Green!25}{grow tomato flowers every week}. \tinycolorbox{Green!25}{They} \tinycolorbox{Green!25}{found a large} \tinycolorbox{Green!25}{arout in the field}. \tinycolorbox{Green!25}{Then she watered them brouted one more time}. \tinycolorbox{Green!25}{Soon} \tinycolorbox{Green!25}{they} \tinycolorbox{Green!25}{could} \tinycolorbox{Green!25}{find three dozen flowers}. \tinycolorbox{Green!25}{They were relieved when the owner} \tinycolorbox{Green!25}{told them} \tinycolorbox{Green!25}{to} \tinycolorbox{Green!25}{go home} \tinycolorbox{Green!25}{safely}.} \\
\midrule
\#1 &
\parbox{.95\linewidth}{\tinycolorbox{red!25}{A local artist} \tinycolorbox{red!25}{wanted to} \tinycolorbox{red!25}{start a new music business}. \tinycolorbox{red!25}{There} \tinycolorbox{red!25}{was a contest that wanted} \tinycolorbox{red!25}{to do a popie}. \tinycolorbox{red!25}{He went on the first show and it was great}. \tinycolorbox{red!25}{When} \tinycolorbox{Green!25}{they} \tinycolorbox{red!25}{arrived to stage} \tinycolorbox{red!25}{the whole crowd felt very sad}. \tinycolorbox{red!25}{The artist} \tinycolorbox{Green!25}{told them} \tinycolorbox{red!25}{to leave, so they left early}.} \\
\#2 &
\parbox{.95\linewidth}{\tinycolorbox{Green!25}{Ava and her} \tinycolorbox{Green!25}{friends} \tinycolorbox{Green!25}{went to} \tinycolorbox{Green!25}{grow tomato flowers every week}. \tinycolorbox{Green!25}{They} \tinycolorbox{Green!25}{found a large} \tinycolorbox{Green!25}{arout in the field}. \tinycolorbox{Green!25}{Then she watered them brouted one more time}. \tinycolorbox{Green!25}{Soon} \tinycolorbox{Green!25}{they} \tinycolorbox{Green!25}{could} \tinycolorbox{Green!25}{find three dozen flowers}. \tinycolorbox{Green!25}{They were relieved when the owner} \tinycolorbox{Green!25}{told them} \tinycolorbox{Green!25}{to} \tinycolorbox{Green!25}{go home} \tinycolorbox{Green!25}{safely}.} \\
\#3 &
\parbox{.95\linewidth}{\tinycolorbox{red!25}{Rita and Roger} \tinycolorbox{red!25}{decided they} \tinycolorbox{red!25}{would own their own farm}. \tinycolorbox{Green!25}{They} \tinycolorbox{red!25}{visited a farm and} \tinycolorbox{red!25}{picked two free lorots}. \tinycolorbox{red!25}{The chef people from the farm were happy to have them}.~\tinycolorbox{Green!25}{They} \tinycolorbox{red!25}{offered to pay for the lorots for a while}. \tinycolorbox{red!25}{The mechanic} \tinycolorbox{Green!25}{told them} \tinycolorbox{red!25}{that they would} \tinycolorbox{red!25}{return for good again}.} \\
\#4 &
\parbox{.95\linewidth}{\tinycolorbox{red!25}{John and his} \tinycolorbox{Green!25}{friends} \tinycolorbox{Green!25}{went to} \tinycolorbox{red!25}{an ice cream shop}. \tinycolorbox{Green!25}{They} \tinycolorbox{red!25}{picked a few chips and sugar} \tinycolorbox{red!25}{candy soda}. \tinycolorbox{red!25}{Then, John started eating a slice}. \tinycolorbox{red!25}{After that slice} \tinycolorbox{Green!25}{they} \tinycolorbox{red!25}{had 10 minutes in} \tinycolorbox{red!25}{one bite}. \tinycolorbox{red!25}{John and his friends walked for a few that day to go shopping and} \tinycolorbox{Green!25}{go home}.} \\
\bottomrule
\end{tabular}
\caption{
Four MBR candidates and the chosen sample for different activation timesteps $t_{\mathrm{act}}$. Parts matching the selected candidate are highlighted in \protect\tinycolorbox{Green!25}{green}; differing parts in \protect\tinycolorbox{red!25}{red}. Parts shared across all candidates are not highlighted.}
\label{tab:mbr_text_example}
\end{table}

%% file: chapters/appendix/q-sampling_ablations.tex
\section{q-sampling ablation}\label{app:q-sampling}

\begin{figure}
    \centering
    \newcolumntype{M}[1]{>{\centering\arraybackslash}m{#1}}
    \setlength{\tabcolsep}{1pt}
    \begin{tabular}{cM{0.31\linewidth}M{0.31\linewidth}M{0.31\linewidth}}
        \midrule
         & & \textbf{ROCStories} & \\
         \midrule
         & Diffusion-LM & TEncDM Emb & TEncDM Enc \\
         \parbox[c]{10pt}{\rotatebox[origin=c]{90}{No SC}} & 
         \includegraphics[width=0.31\textwidth]{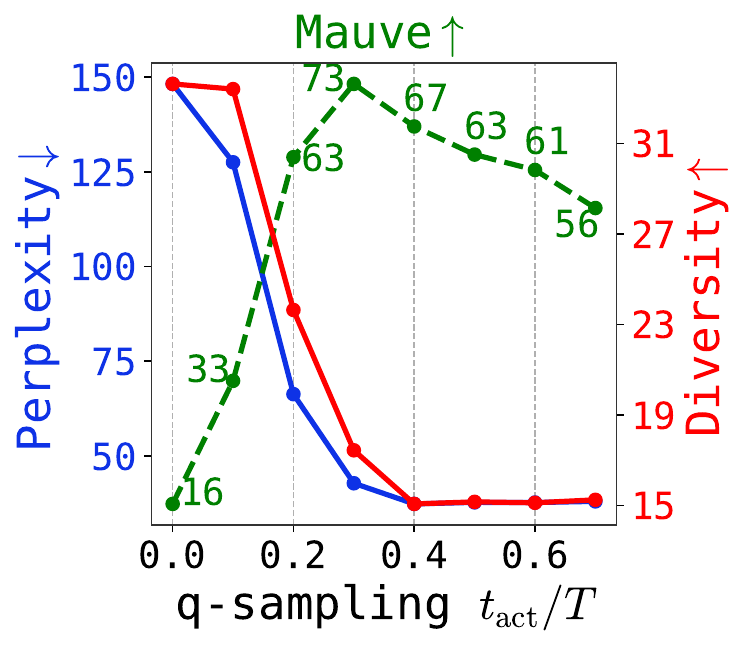} &
         \includegraphics[width=0.31\textwidth]{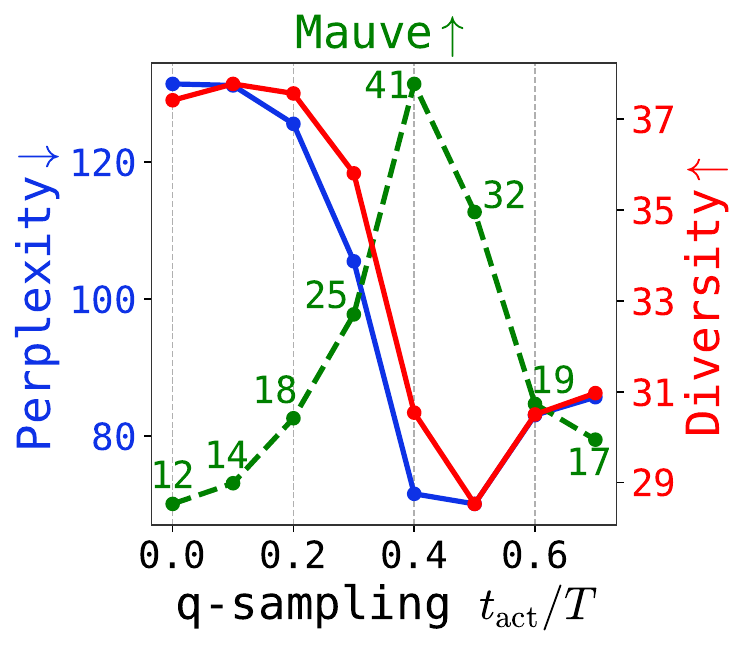} &
         \includegraphics[width=0.31\textwidth]{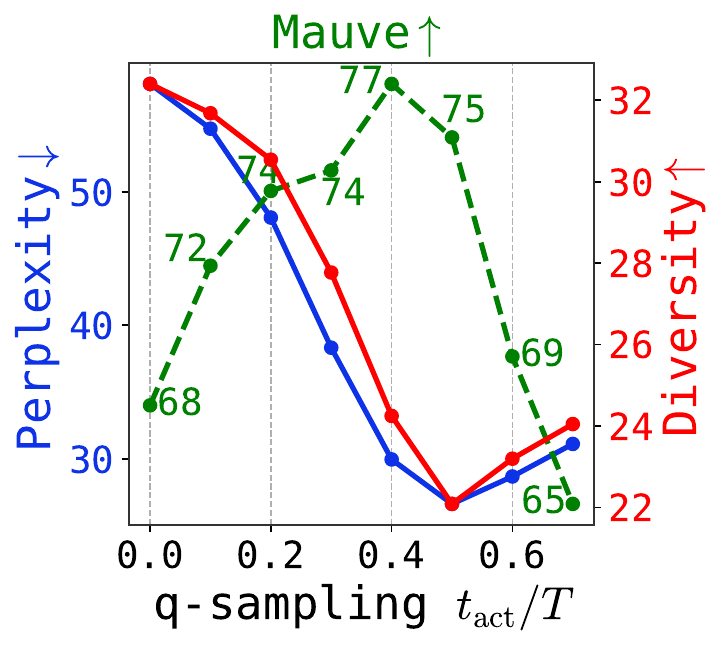} \\
         \parbox[c]{10pt}{\rotatebox[origin=c]{90}{With SC}} & 
         \includegraphics[width=0.31\textwidth]{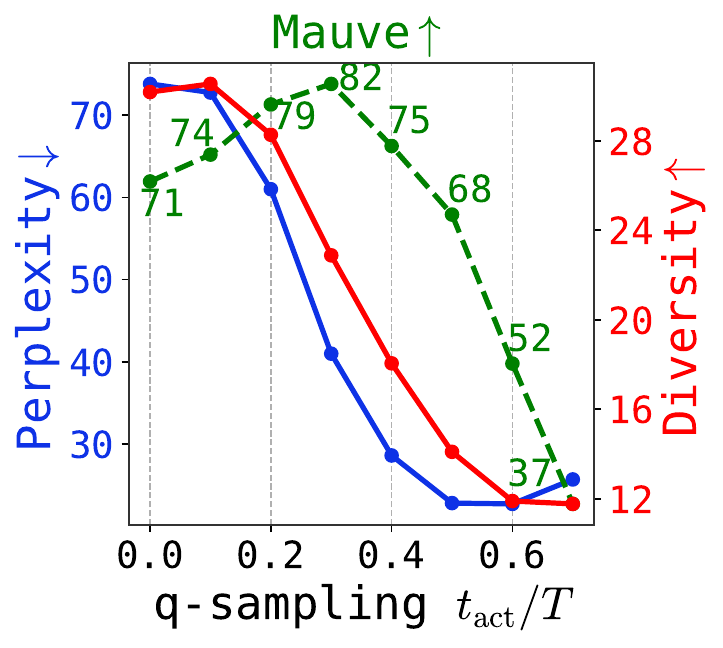} &
         \includegraphics[width=0.31\textwidth]{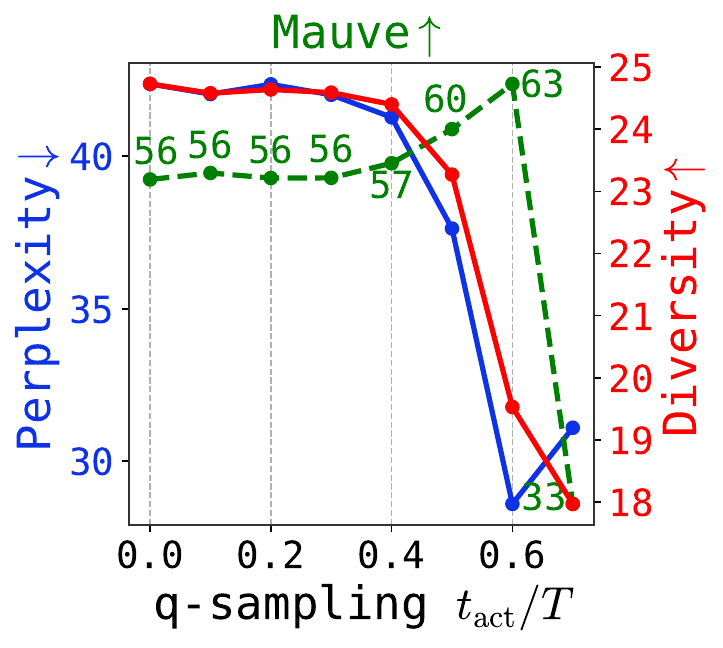} &
         \includegraphics[width=0.31\textwidth]{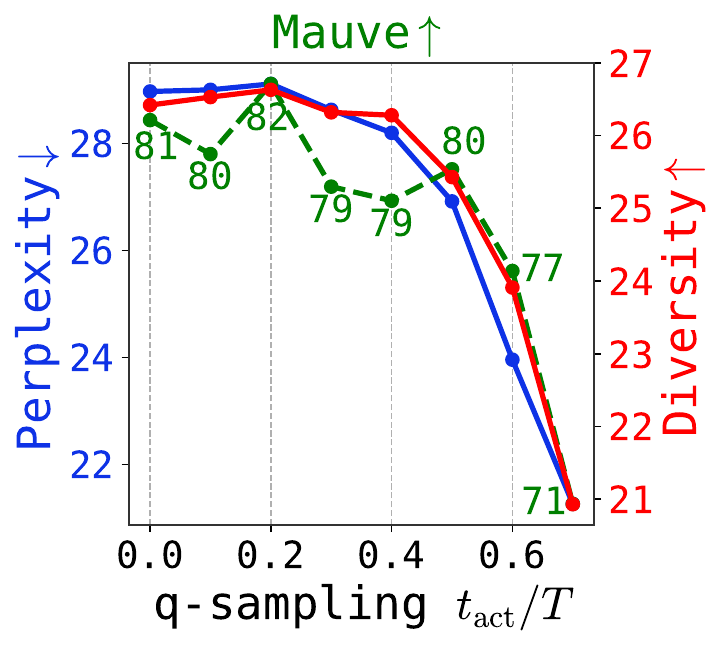} \\
         \midrule
         & & \textbf{Wikipedia} & \\
         \midrule
         & Diffusion-LM & TEncDM Emb & TEncDM Enc \\
         \parbox[c]{10pt}{\rotatebox[origin=c]{90}{No SC}} & 
         \includegraphics[width=0.31\textwidth]{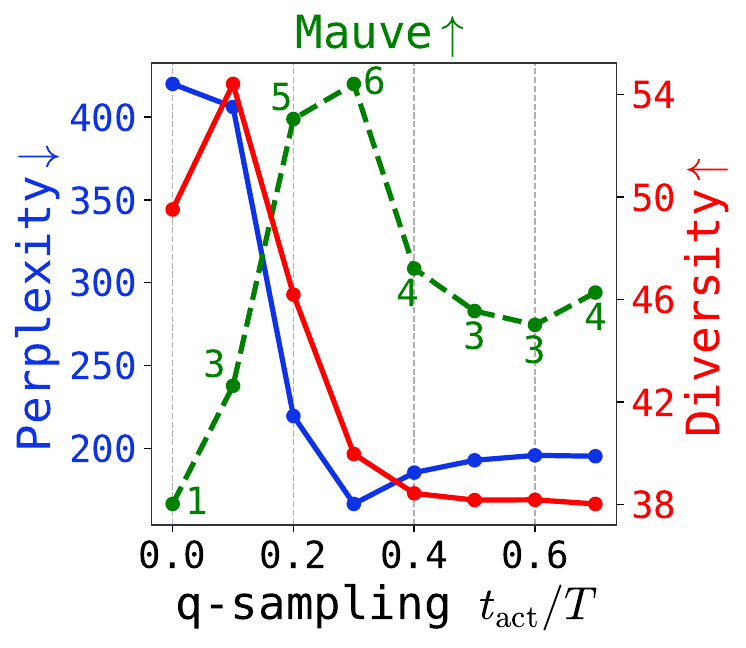} &
         \includegraphics[width=0.31\textwidth]{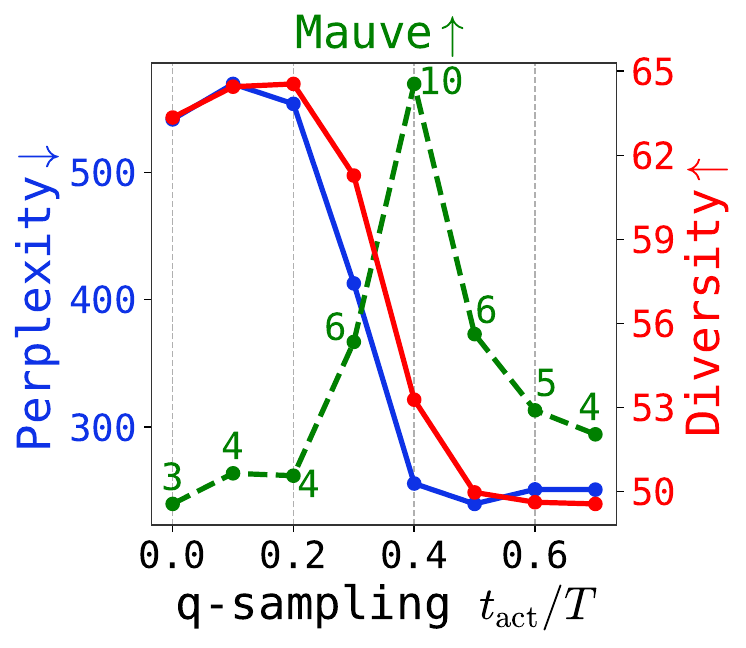} &
         \includegraphics[width=0.31\textwidth]{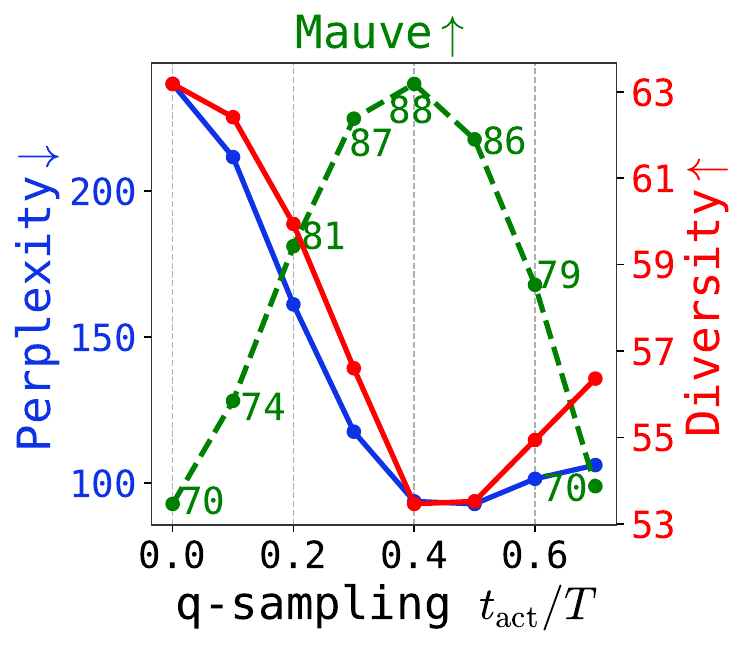} \\
         \parbox[c]{10pt}{\rotatebox[origin=c]{90}{With SC}} & 
         \includegraphics[width=0.31\textwidth]{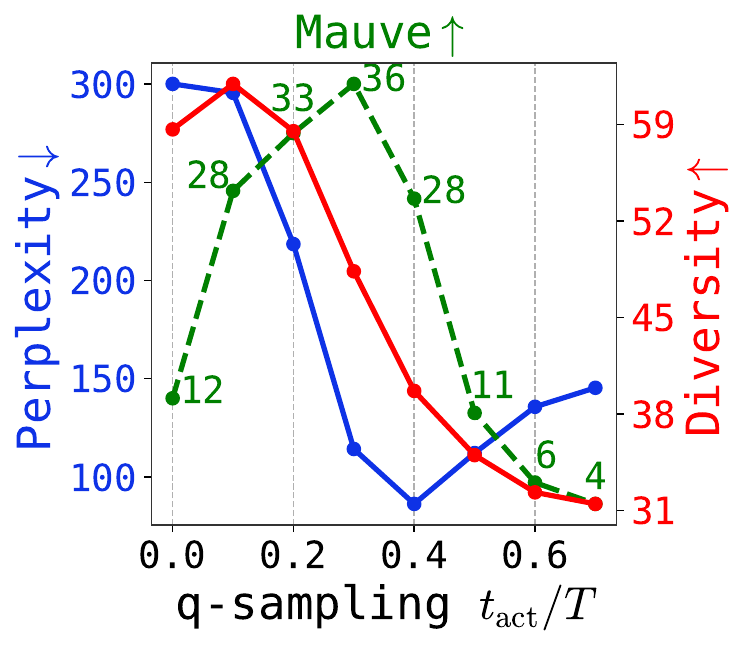} &
         \includegraphics[width=0.31\textwidth]{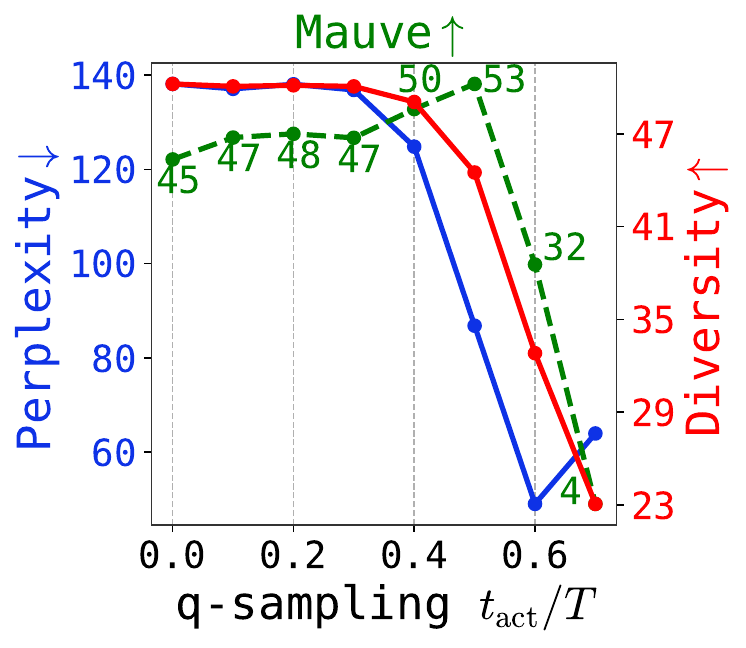} &
         \includegraphics[width=0.31\textwidth]{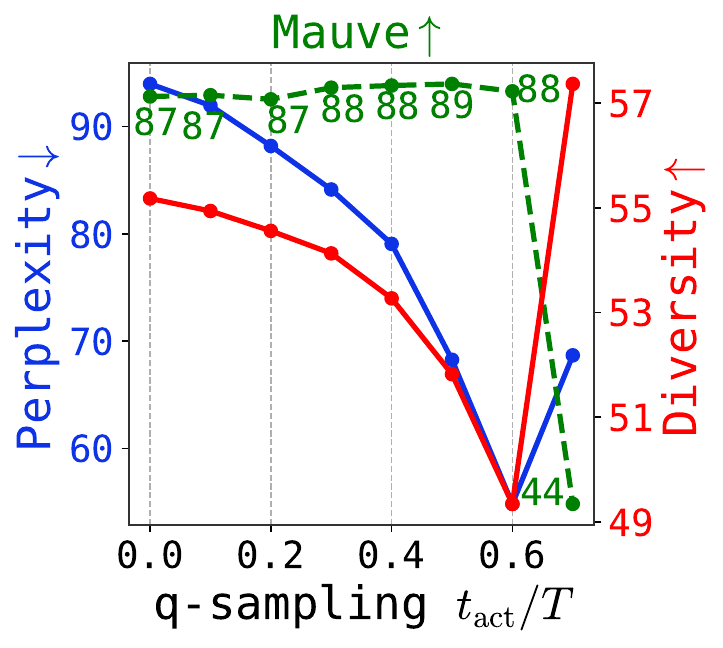} \\
    \end{tabular}
    \caption{Quality metrics with and without self-conditoning (SC) for different q-sampling activation timestep on ROCStories and Wikipedia. We use the baseline SC model with $p=0.5$.}
    \label{fig:q_sampling_ablation}
\end{figure}

\begin{figure}
    \centering
    \newcolumntype{M}[1]{>{\centering\arraybackslash}m{#1}}
    \setlength{\tabcolsep}{1pt}
    \begin{tabular}{cM{0.31\linewidth}M{0.31\linewidth}}
        \midrule
         & \multicolumn{2}{M{0.62\linewidth}}{\textbf{OpenWebText}} \\
         \midrule
         & TEncDM Emb & TEncDM Enc \\
         \parbox[c]{10pt}{\rotatebox[origin=c]{90}{No SC}} & 
         \includegraphics[width=0.31\textwidth]{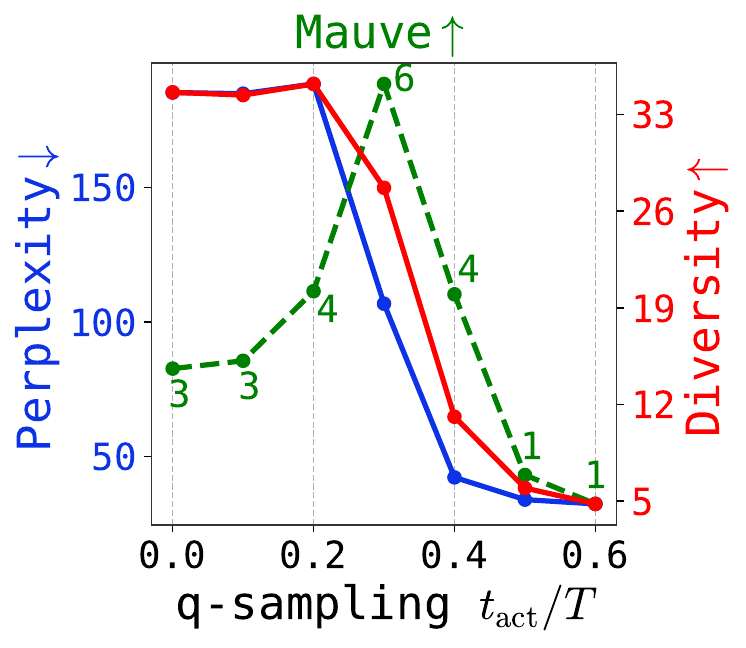} &
         \includegraphics[width=0.31\textwidth]{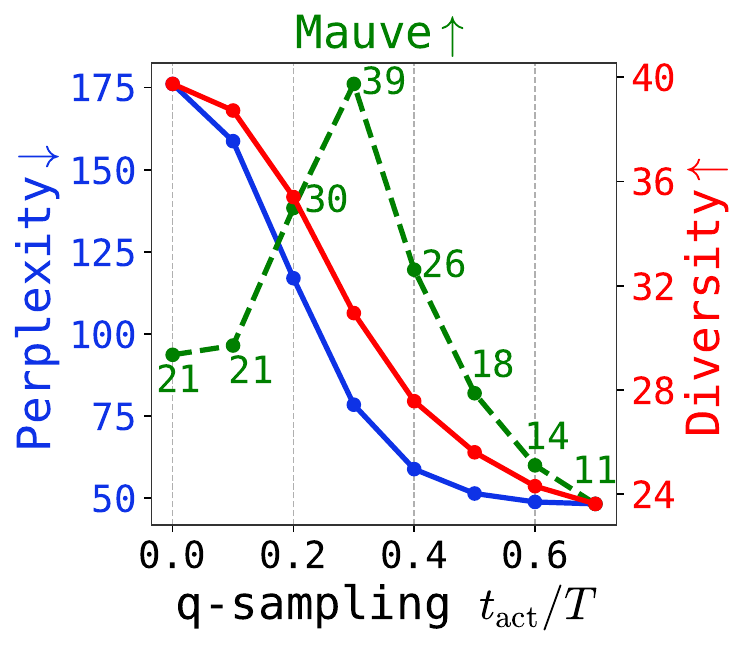} \\
         \parbox[c]{10pt}{\rotatebox[origin=c]{90}{With SC}} & 
         \includegraphics[width=0.31\textwidth]{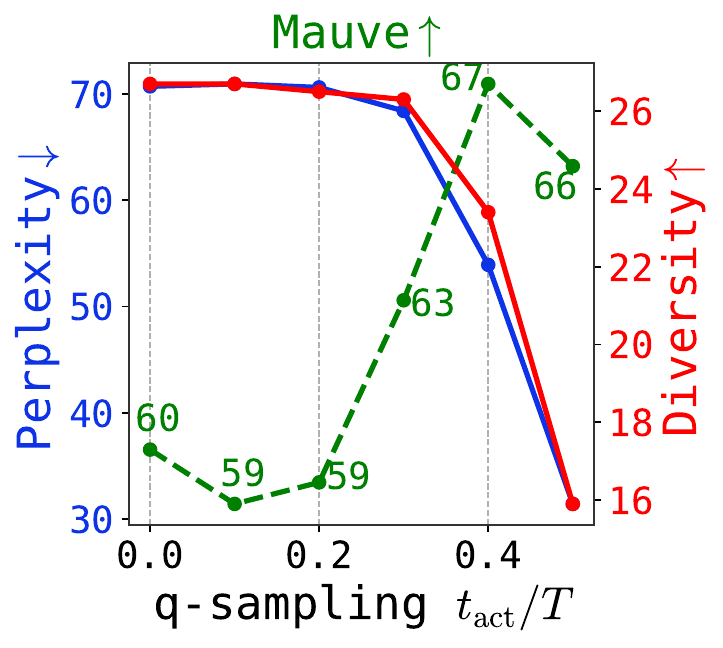} &
         \includegraphics[width=0.31\textwidth]{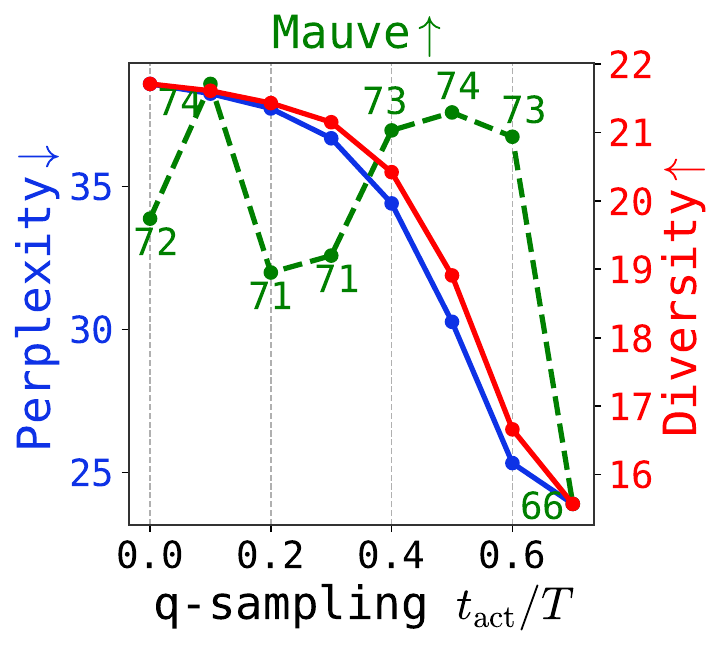} \\
    \end{tabular}
    \caption{Quality metrics with and without self-conditoning (SC) for different q-sampling activation timestep on OpenWebText. We use the baseline SC model with $p=0.5$.}
    \label{fig:q_sampling_owt}
\end{figure}

In this section, we analyze how generation quality varies with the q-sampling activation timestep $t_{\mathrm{act}}$.
Results are presented in Figures~\ref{fig:q_sampling_ablation} and~\ref{fig:q_sampling_owt}.
Overall, enabling self-conditioning consistently improves Mauve and reduces perplexity while also decreasing diversity, which is consistent across all datasets and models.
Regarding the effect of $t_{\mathrm{act}}$, two trends emerge.
First, activating q-sampling earlier (larger $t_{\mathrm{act}}/T$) mostly monotonically degrades both perplexity and diversity.
Second, Mauve exhibits a clear peak as a function of $t_{\mathrm{act}}$.

Without SC, this peak is pronounced and consistent across datasets: $t_{\mathrm{act}}/T = 0.3$ for Diffusion-LM and $t_{\mathrm{act}}/T \approx 0.3$--$0.4$ for TEncDM Emb and TEncDM Enc.
With SC, pure DDPM ($t_{\mathrm{act}}/T = 0$) already provides a strong baseline (except for Diffusion-LM on Wikipedia), so the relative gain from q-sampling is smaller, though the optimal threshold shifts slightly: $t_{\mathrm{act}}/T = 0.3$ for Diffusion-LM and $t_{\mathrm{act}}/T = 0.5$ for TEncDM Emb.
For TEncDM Enc, q-sampling yields only a marginal improvement over pure DDPM regardless of $t_{\mathrm{act}}$, which is consistent with the intuition that context-dependent encodings already produce a smoother latent space more suitable for Gaussian diffusion.
In all settings, activating q-sampling too early ($t_{\mathrm{act}}/T = 0.7$) leads to degraded performance, in line with our findings from the RHM experiments.

%% file: chapters/appendix/number_of_steps.tex
\section{Changing the number of steps}\label{app:number_of_steps}

In this section, we provide a qualitative comparison of the TEncDM Emb and TEncDM Enc models when using different numbers of generation steps at inference. We choose TEncDM-based models because they do not require retraining to change the number of steps.

In Figure~\ref{fig:steps:without_sc} and Figure~\ref{fig:steps:with_sc}, we show how varying the number of steps affects generation quality for models without and with SC, respectively. For models with SC, increasing the number of steps significantly reduces both perplexity and diversity, regardless of whether q-sampling is used. This is consistent with the findings of \citet{tencdm} and can be explained by the growing SC mismatch discussed in Appendix~\ref{app:related_work}.

The results for models without SC are more nuanced. In the DDPM setting ($t_{\mathrm{act}} / T = 0$), the difference between runs with different numbers of steps is marginal. However, as the q-sampling activation interval increases, the gap widens: perplexity and diversity drop more rapidly with more steps.

This behavior can be explained using the analysis in Appendix~\ref{app:prediction_variance}. The q-sampling solver reduces the variance of $\mathbf{x}_{t-1}$ at each update compared to DDPM. Consequently, the more updates that are performed with q-sampling, the lower the resulting variance. With larger number of steps, the final latent ends up closer to the center of the distribution, meaning it is less committed to any particular text. This leads to lower diversity and perplexity, as the two are correlated.

When DDPM perplexity is high, as for TEncDM Emb, such a reduction in perplexity is beneficial and Mauve increases. However, when perplexity is already low, as for TEncDM Enc, further reduction hurts Mauve.

\begin{figure}
    \centering
    \newcolumntype{M}[1]{>{\centering\arraybackslash}m{#1}}
    \setlength{\tabcolsep}{1pt}
    \begin{tabular}{c}
         \midrule
         \textbf{TEncDM Emb (Without SC)} \\
         \midrule
         ROCStories \\
         \includegraphics[width=0.95\textwidth]{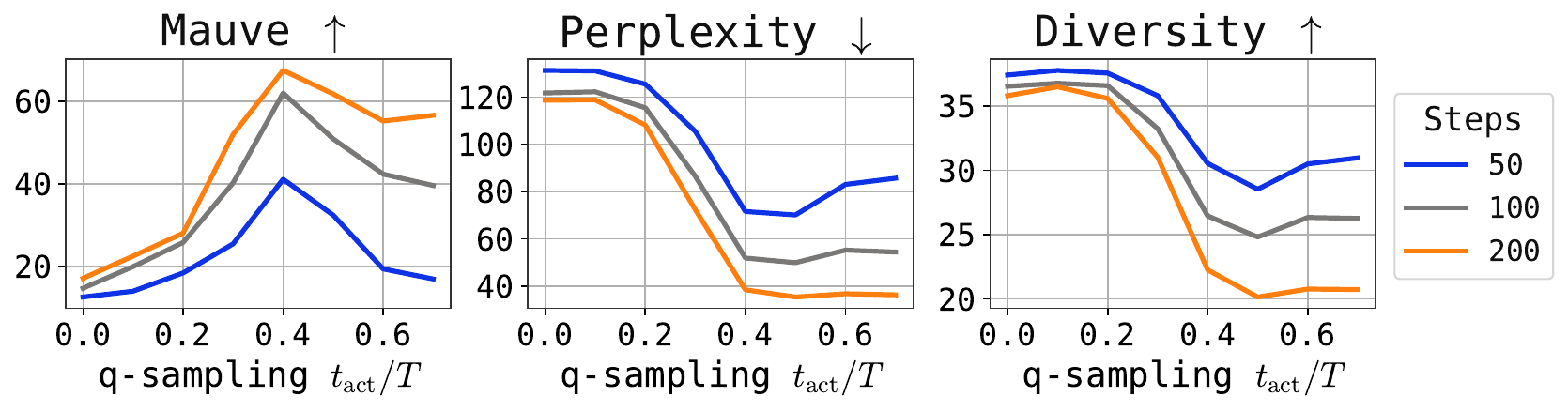} \\
         Wikipedia \\
         \includegraphics[width=0.95\textwidth]{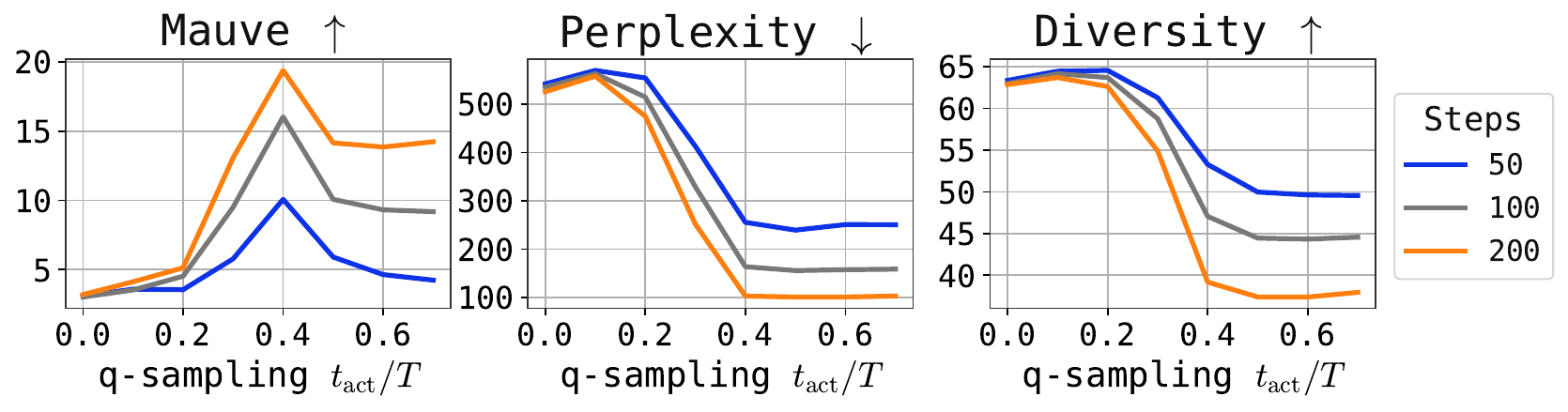} \\
         \midrule
         \textbf{TEncDM Enc (Without SC)} \\
         \midrule
         ROCStories \\
         \includegraphics[width=0.95\textwidth]{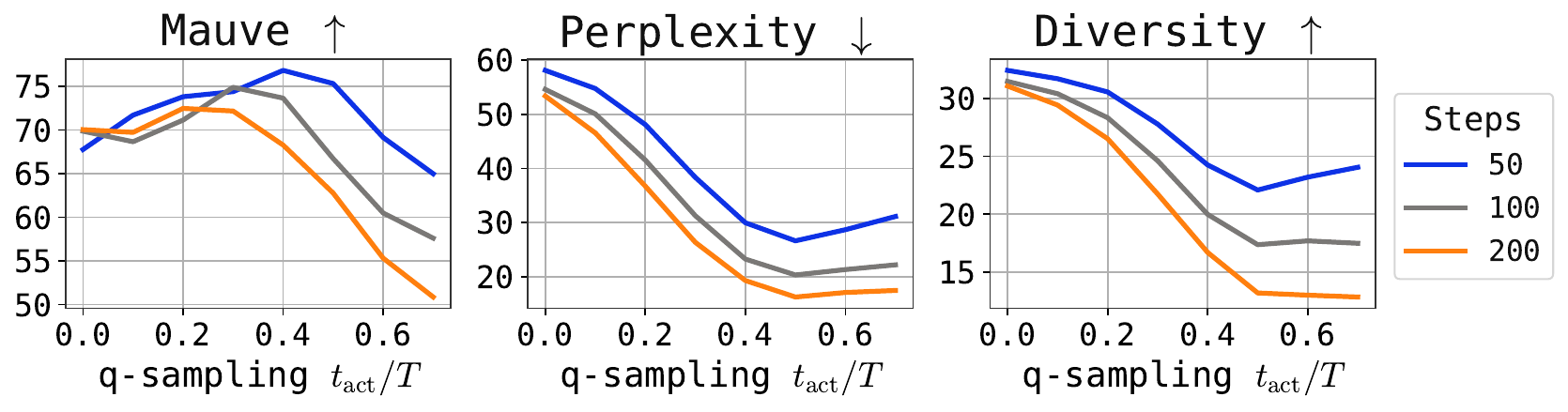} \\
         Wikipedia \\
         \includegraphics[width=0.95\textwidth]{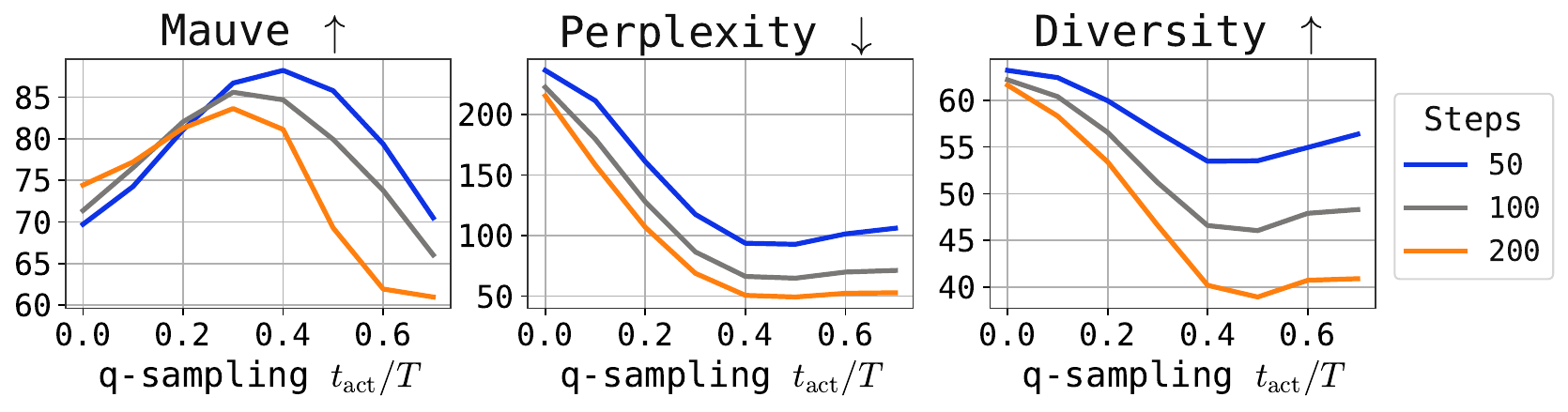} \\
    \end{tabular}
    \caption{Quality metrics for TEncDM-based models \textbf{without self-conditining (SC)} for different q-sampling activation timesteps and varying amount of generation steps on ROCStories and Wikipedia. We use the baseline SC model with $p=0.5$.}
    \label{fig:steps:without_sc}
\end{figure}

\begin{figure}
    \centering
    \newcolumntype{M}[1]{>{\centering\arraybackslash}m{#1}}
    \setlength{\tabcolsep}{1pt}
    \begin{tabular}{c}
         \midrule
         \textbf{TEncDM Emb (With SC)} \\
         \midrule
         ROCStories \\
         \includegraphics[width=0.95\textwidth]{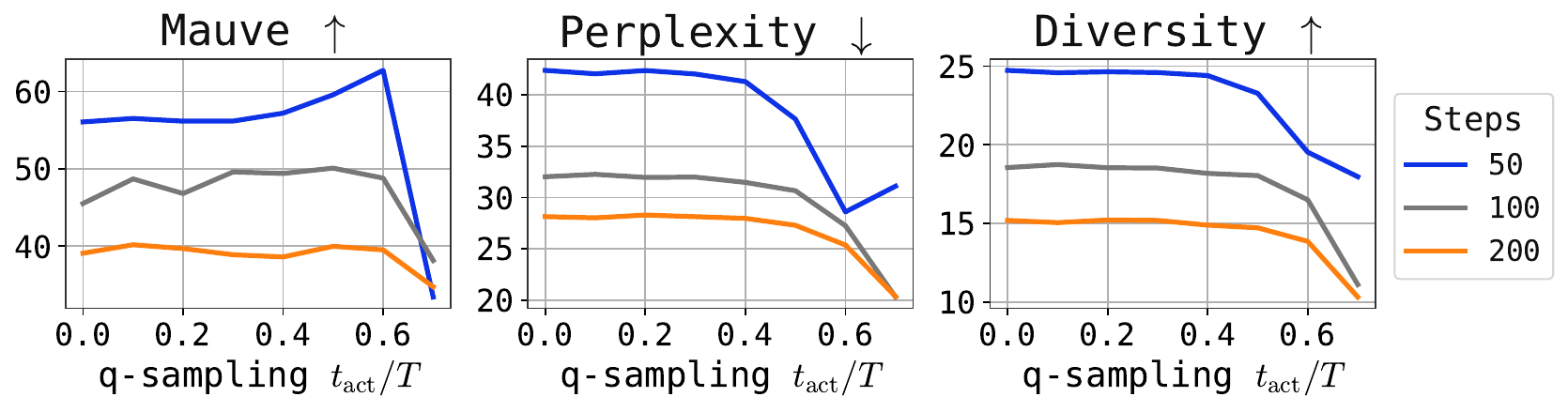} \\
         Wikipedia \\
         \includegraphics[width=0.95\textwidth]{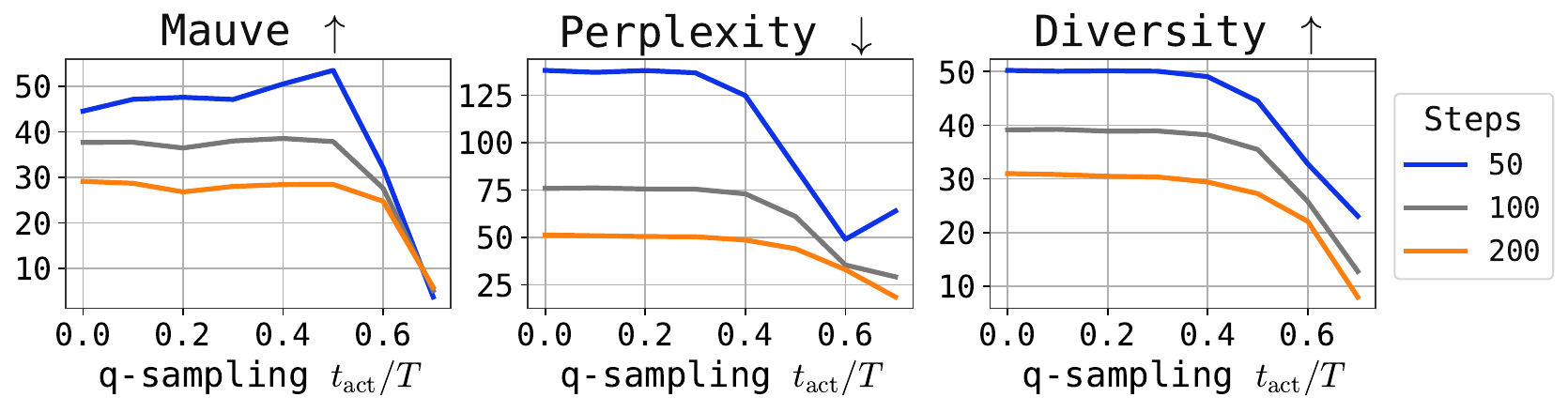} \\
         \midrule
         \textbf{TEncDM Enc (With SC)} \\
         \midrule
         ROCStories \\
         \includegraphics[width=0.95\textwidth]{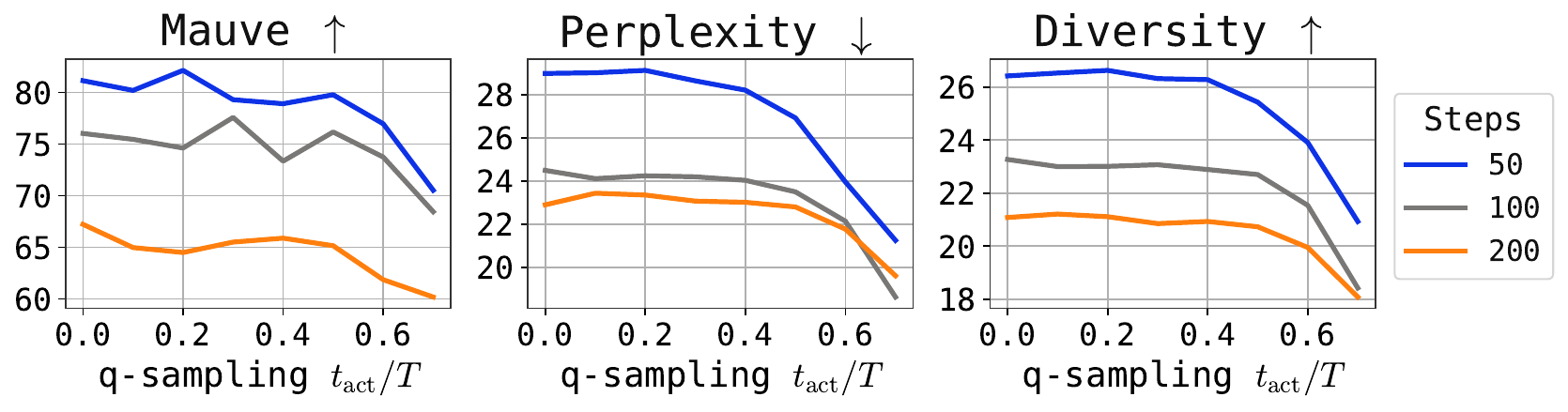} \\
         Wikipedia \\
         \includegraphics[width=0.95\textwidth]{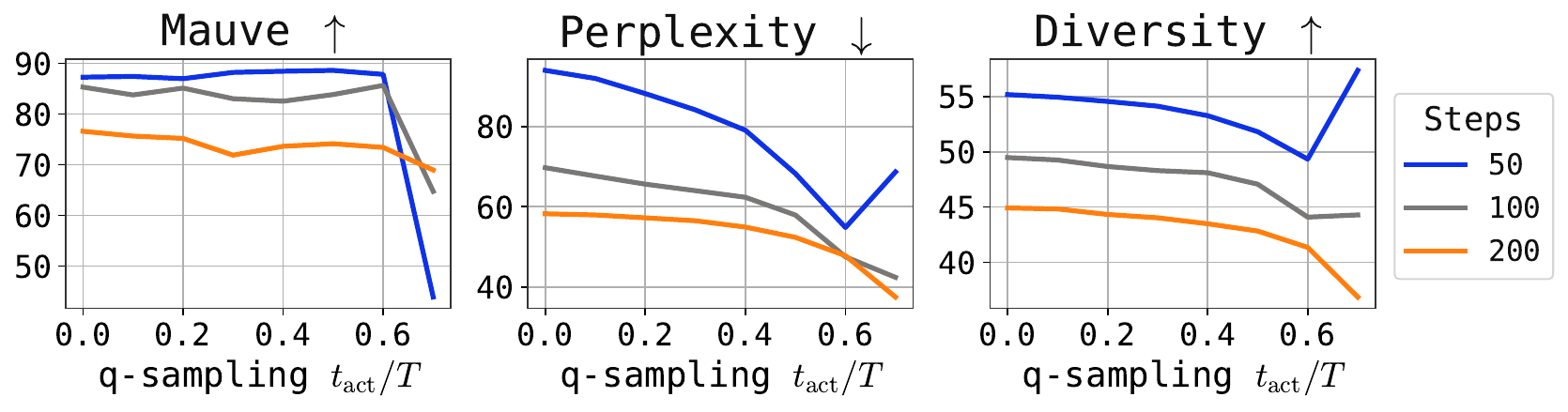} \\
    \end{tabular}
    \caption{Quality metrics for TEncDM-based models \textbf{with self-conditining (SC)} for different q-sampling activation timesteps and varying amount of generation steps on ROCStories and Wikipedia. We use the baseline SC model with $p=0.5$.}
    \label{fig:steps:with_sc}
\end{figure}

%% file: chapters/appendix/sc_prob.tex
\section{Self-conditioning probability}\label{app:sc_prob}

In this section, we ablate the probability $p$ of self-conditioning (SC) used during training. We argue that the commonly used value of $p = 0.5$ is suboptimal: since SC is applied at every iteration during inference, it may be beneficial to expose the model to conditioning more frequently during training.

In Table \ref{tab:sc_ablation}, we report a numerical evaluation across varying values of $p$ and make two observations. First, increasing $p$ generally decreases perplexity, while diversity initially rises and then falls. This suggests that overall text quality improves as $p$ increases, up to a point. This trend is further supported by the Mauve scores, which peak at an intermediate value of $p$ around $0.9$. Second, the optimal value of $p$ is difficult to identify, as it varies across models and datasets. As a rule of thumb, however, we recommend setting $p$
to at least $0.8$ and tuning upward from there.

\begin{table}
\small
\centering
\setlength{\tabcolsep}{3.2pt}
\begin{tabular}{l|ccc|ccc|ccc}
\toprule
& \multicolumn{3}{c|}{\textbf{ROCStories}} & \multicolumn{3}{c|}{\textbf{Wikipedia}} & \multicolumn{3}{c}{\textbf{OpenWebText}}\\
\textbf{Model} & \textbf{Mauve} $\uparrow$ & \textbf{PPL} $\downarrow$ & \textbf{Div} $\uparrow$  & \textbf{Mauve} $\uparrow$ & \textbf{PPL} $\downarrow$ & \textbf{Div} $\uparrow$ & \textbf{Mauve} $\uparrow$ & \textbf{PPL} $\downarrow$ & \textbf{Div} $\uparrow$ \\
\midrule
\multicolumn{4}{l}{\textit{Diffusion-LM}} \\
\, + SC$_{0.5}$ & $71.3_{\gray{\pm 3.8}}$ & $74.5_{\gray{\pm 0.9}}$ & $\textbf{30.2}_{\gray{\pm 0.3}}$ & $12.1_{\gray{\pm 1.9}}$ & $300.2_{\gray{\pm 6.7}}$ & $58.7_{\gray{\pm 0.4}}$ & — & — & —
 \\
\, + SC$_{0.8}$ & $74.0_{\gray{\pm 4.3}}$ & $66.5_{\gray{\pm 1.0}}$ & $30.1_{\gray{\pm 0.4}}$ & $16.8_{\gray{\pm 1.6}}$ & $271.2_{\gray{\pm 7.4}}$ & $\textbf{58.8}_{\gray{\pm 0.5}}$ & — & — & —
 \\
\, + SC$_{0.9}$ & $74.0_{\gray{\pm 1.9}}$ & $61.8_{\gray{\pm 0.5}}$ & $29.2_{\gray{\pm 0.3}}$ & $24.3_{\gray{\pm 3.5}}$ & $229.5_{\gray{\pm 5.1}}$ & $\textbf{58.8}_{\gray{\pm 0.5}}$ & — & — & —
 \\
\, + SC$_{0.95}$ & \cellcolor{blue!15}$\textbf{75.2}_{\gray{\pm 2.4}}$ & \cellcolor{blue!15}$59.5_{\gray{\pm 1.1}}$ & \cellcolor{blue!15}$29.0_{\gray{\pm 0.4}}$ & $27.8_{\gray{\pm 3.2}}$ & $197.4_{\gray{\pm 4.8}}$ & $56.6_{\gray{\pm 0.5}}$ & — & — & —
 \\
\, + SC$_{0.99}$ & $74.8_{\gray{\pm 3.4}}$ & $\textbf{58.7}_{\gray{\pm 1.0}}$ & $28.8_{\gray{\pm 0.3}}$ & \cellcolor{blue!15}$\textbf{31.0}_{\gray{\pm 4.2}}$ & \cellcolor{blue!15}$\textbf{158.0}_{\gray{\pm 6.2}}$ & \cellcolor{blue!15}$48.6_{\gray{\pm 1.3}}$ & — & — & —
 \\
\midrule
\multicolumn{4}{l}{\textit{TEncDM Embeddings}} \\
\, + SC$_{0.5}$ & $56.5_{\gray{\pm 4.2}}$ & $42.4_{\gray{\pm 0.8}}$ & $24.8_{\gray{\pm 0.3}}$ & $44.5_{\gray{\pm 4.6}}$ & $138.1_{\gray{\pm 2.3}}$ & $50.2_{\gray{\pm 0.4}}$ & \cellcolor{blue!15}$\textbf{59.9}_{\gray{\pm 4.7}}$ & \cellcolor{blue!15}$70.7_{\gray{\pm 1.0}}$ & \cellcolor{blue!15}$\textbf{26.7}_{\gray{\pm 0.3}}$
 \\
\, + SC$_{0.8}$ & $73.2_{\gray{\pm 4.3}}$ & $37.7_{\gray{\pm 0.6}}$ & $26.2_{\gray{\pm 0.5}}$ & \cellcolor{blue!15}$\textbf{51.1}_{\gray{\pm 3.9}}$ & \cellcolor{blue!15}$129.6_{\gray{\pm 3.0}}$ & \cellcolor{blue!15}$51.7_{\gray{\pm 0.6}}$ & $48.6_{\gray{\pm 4.5}}$ & $43.9_{\gray{\pm 0.5}}$ & $17.4_{\gray{\pm 0.2}}$
 \\
\, + SC$_{0.9}$ & $73.6_{\gray{\pm 3.6}}$ & $37.6_{\gray{\pm 0.5}}$ & $\textbf{28.7}_{\gray{\pm 0.4}}$ & $50.1_{\gray{\pm 3.6}}$ & $126.5_{\gray{\pm 4.3}}$ & $\textbf{52.9}_{\gray{\pm 0.9}}$ & $35.4_{\gray{\pm 4.4}}$ & $36.9_{\gray{\pm 0.5}}$ & $13.9_{\gray{\pm 0.3}}$
 \\
\, + SC$_{0.95}$ & \cellcolor{blue!15}$\textbf{78.9}_{\gray{\pm 3.0}}$ & \cellcolor{blue!15}$\textbf{36.0}_{\gray{\pm 0.5}}$ & \cellcolor{blue!15}$27.7_{\gray{\pm 0.3}}$ & $46.5_{\gray{\pm 4.9}}$ & $131.6_{\gray{\pm 2.9}}$ & $51.7_{\gray{\pm 0.6}}$ & $34.8_{\gray{\pm 4.2}}$ & $\textbf{32.9}_{\gray{\pm 0.7}}$ & $12.8_{\gray{\pm 0.4}}$
 \\
\, + SC$_{0.99}$ & $78.8_{\gray{\pm 2.7}}$ & $36.7_{\gray{\pm 0.6}}$ & $27.0_{\gray{\pm 0.4}}$ & $43.1_{\gray{\pm 3.5}}$ & $\textbf{116.3}_{\gray{\pm 2.5}}$ & $50.5_{\gray{\pm 0.7}}$ & — & — & — \\
\midrule
\multicolumn{4}{l}{\textit{TEncDM Encodings}}  \\
\, + SC$_{0.5}$ & $81.0_{\gray{\pm 2.9}}$ & $29.1_{\gray{\pm 0.5}}$ & $26.5_{\gray{\pm 0.4}}$ & $87.3_{\gray{\pm 3.0}}$ & $94.0_{\gray{\pm 1.8}}$ & $55.2_{\gray{\pm 0.3}}$ & $71.6_{\gray{\pm 3.8}}$ & $\textbf{38.6}_{\gray{\pm 0.4}}$ & $21.7_{\gray{\pm 0.2}}$
 \\
\, + SC$_{0.8}$ & $81.1_{\gray{\pm 2.3}}$ & $30.7_{\gray{\pm 0.3}}$ & $28.4_{\gray{\pm 0.4}}$ & \cellcolor{blue!15}$\textbf{91.4}_{\gray{\pm 1.3}}$ & \cellcolor{blue!15}$\textbf{90.2}_{\gray{\pm 1.5}}$ & \cellcolor{blue!15}$57.0_{\gray{\pm 0.4}}$ & $64.3_{\gray{\pm 4.5}}$ & $47.1_{\gray{\pm 0.4}}$ & $24.4_{\gray{\pm 0.2}}$
 \\
\, + SC$_{0.9}$ & $84.1_{\gray{\pm 2.8}}$ & $\textbf{29.0}_{\gray{\pm 0.5}}$ & $31.0_{\gray{\pm 0.4}}$ & $90.8_{\gray{\pm 1.8}}$ & $90.8_{\gray{\pm 1.4}}$ & $\textbf{58.3}_{\gray{\pm 0.4}}$ & \cellcolor{blue!15}$\textbf{72.4}_{\gray{\pm 3.8}}$ & \cellcolor{blue!15}$49.2_{\gray{\pm 0.5}}$ & \cellcolor{blue!15}$26.5_{\gray{\pm 0.2}}$
 \\
\, + SC$_{0.95}$ & \cellcolor{blue!15}$\textbf{86.8}_{\gray{\pm 1.9}}$ & \cellcolor{blue!15}$29.9_{\gray{\pm 0.3}}$ & \cellcolor{blue!15}$\textbf{31.5}_{\gray{\pm 0.4}}$ & $90.1_{\gray{\pm 2.1}}$ & $92.6_{\gray{\pm 1.6}}$ & $57.3_{\gray{\pm 0.3}}$ & $67.1_{\gray{\pm 3.2}}$ & $50.7_{\gray{\pm 0.6}}$ & $\textbf{26.6}_{\gray{\pm 0.2}}$
 \\
\, + SC$_{0.99}$ & $85.2_{\gray{\pm 1.5}}$ & $29.8_{\gray{\pm 0.5}}$ & $31.1_{\gray{\pm 0.4}}$ & $86.4_{\gray{\pm 2.5}}$ & $93.6_{\gray{\pm 1.6}}$ & $56.1_{\gray{\pm 0.3}}$ & — & — & —
 \\
\bottomrule
\end{tabular}
\caption{Impact of self-conditioning (SC) probability on ROCStories, Wikipedia and OpenWebText. SC trained with probability $p$ is denoted SC$_{p}$. The SC probability giving the highest Mauve is highlighted in \protect\tinycolorbox{blue!15}{blue}.} 
\label{tab:sc_ablation}
\end{table}

%% file: chapters/appendix/unsuccessful_ideas.tex
\section{Unsuccessful ideas}\label{app:ideas}

We conducted preliminary experiments on ROCStories exploring several ideas that did not ultimately improve generation quality, but which we include here as potentially useful context for the reader.

\paragraph{Adding variance to $\mathbf{x}_0$}
A key finding of this paper is that more continuous $\mathbf{x}_0$ distributions yield better generation quality.
Although unnecessary in principle, methods that train embeddings jointly with the diffusion model typically add a small variance to $\mathbf{x}_0$, making its distribution continuous rather than a delta-mixture.
We hypothesized that this smoothing might reduce low-density regions and improve generation.
While experiments on RHM supported this idea, the effect did not transfer to real data.
\citet{dinoiser} reached a similar conclusion.

\paragraph{Training embeddings jointly with diffusion}
We hypothesized that since learnable embeddings evolve throughout training, they introduce additional variance into $\mathbf{x}_0$, effectively making the latent space more continuous.
To test this, we trained one model with learnable embeddings and another with frozen embeddings extracted from the first model, ensuring both models operate on identical embeddings at inference.
As a result, the two models produced equivalent quality.
We suppose that embeddings converge early in training; they do not vary much across iterations, and therefore the diffusion model does not perceive the latents as elements of the continuous space, leading to no improvement.

\paragraph{Resampling $\mathbf{x}_t$ for self-conditioning}
During training, SC is introduced by first computing a prediction with the SC input set to zero, $\bar{\mathbf{x}}^{t}_0 = \hat{\mathbf{x}}_{\theta}(\mathbf{x}_t, t, \mathbf{0})$, and then using this as the SC input for a second prediction, $\hat{\mathbf{x}}^{t}_0 = \hat{\mathbf{x}}_{\theta}(\mathbf{x}_t, t, \operatorname{SG}(\bar{\mathbf{x}}^t_0))$.
Crucially, $\mathbf{x}_t$ is held fixed between the two predictions during training, whereas during generation it updates between steps.
We hypothesized that this discrepancy could hurt generation quality, particularly for q-sampling where the latent changes substantially between iterations.
To address this, we experimented with resampling $\mathbf{x}_t$ for the second prediction during training with some probability $p'$, and additionally providing the model with a flag indicating which type of SC input was used.
Neither modification improved generation quality.

%% file: chapters/appendix/unconditional.tex
\section{Additional results for unconditional generation}\label{app:unconditional}

\begin{wraptable}{r}{0.51\textwidth}
\vspace{-1em}
% \fontsize{8}{9}\selectfont
\small
\centering
\setlength{\tabcolsep}{3.2pt}
\begin{tabular}{l|ccc}
\toprule
 & \multicolumn{3}{c}{\textbf{OpenWebText}} \\
Model & \textbf{Mauve} $\uparrow$ & \textbf{PPL} $\downarrow$ & \textbf{Div} $\uparrow$ \\
\midrule
TEncDM Emb & $2.6_{\gray{\pm 0.4}}$ & $185.3_{\gray{\pm 2.4}}$ & $\textbf{34.6}_{\gray{\pm 0.2}}$ \\
\, + SC$_{0.5}$ & $59.8_{\gray{\pm 4.7}}$ & $70.6_{\gray{\pm 0.8}}$ & $26.7_{\gray{\pm 0.3}}$ \\
\, + QS$_{0.4}$ & $6.2_{\gray{\pm 0.8}}$ & $106.8_{\gray{\pm 1.7}}$ & $27.7_{\gray{\pm 0.3}}$ \\
\, + QS$_{0.4}$ + SC$_{0.5}$ & 
$\textbf{67.3}_{\gray{\pm 4.5}}$ & $\textbf{53.9}_{\gray{\pm 0.7}}$ & $23.4_{\gray{\pm 0.3}}$ \\
\midrule
TEncDM Enc & $20.8_{\gray{\pm 3.0}}$ & $176.1_{\gray{\pm 2.0}}$ & $\textbf{39.7}_{\gray{\pm 0.3}}$ \\
\, + SC$_{0.5}$ & $71.6_{\gray{\pm 3.8}}$ & $38.6_{\gray{\pm 0.4}}$ & $21.7_{\gray{\pm 0.2}}$ \\
\, + QS$_{0.4}$ & $38.5_{\gray{\pm 3.9}}$ & $78.4_{\gray{\pm 1.1}}$ & $30.9_{\gray{\pm 0.3}}$ \\
\, + QS$_{0.4}$ + SC$_{0.5}$ & $\textbf{73.2}_{\gray{\pm 3.8}}$ & $\textbf{34.4}_{\gray{\pm 0.3}}$ & $20.4_{\gray{\pm 0.2}}$ \\
\bottomrule
\end{tabular}
\caption{Impact of q-sampling (QS) and self-conditioning (SC) on OpenWebText. SC trained with probability $p$ is denoted SC$_{p}$; QS with optimal activation timestep $t_{\mathrm{act}}$ is denoted QS$_{t_{\mathrm{act}}/T}$.} 
\label{tab:openwebtext}
\vspace{-3em}
\end{wraptable}

In this section, we present additional results for the unconditional generation. Namely, we conduct the experiments on the large scale OpenWebText dataset to show that our theory extends to longer sequences, and on protein generation task to extend our results to discrete domains beyond text.

\subsection{OpenWebText results}\label{app:openwebtext}

In Table~\ref{tab:openwebtext}, we evaluate the impact of q-sampling and self-conditioning on the OpenWebText dataset. The text length for this task is $512$, which is 4 times larger that $128$ used for Wikipedia. 
The results are fully consistent with those in Table~\ref{tab:q-sampling_sc} for other datasets, though SC proves especially important on longer sequences, yielding a substantial boost in performance. 
Applying q-sampling within the critical interval also improves quality both with and without SC, confirming that it addresses the DDPM limitations discussed in Section~\ref{sec:mitigation}.
As in the other experiments, both SC and q-sampling reduce the diversity of generated texts in exchange for lower perplexity.

\subsection{Protein generation results}\label{app:protein}

\begin{wraptable}{r}{0.51\textwidth}
\vspace{-3em}
%\fontsize{8}{9}\selectfont
\small
\centering
\setlength{\tabcolsep}{3.5pt}
\begin{tabular}{l|ccc}
\toprule
 & \textbf{FD-seq} $\downarrow$ & \textbf{pppl} $\downarrow$ & \textbf{CD$_{0.5}$} $\uparrow$ \\
\midrule
DiMA Emb & $1.086_{\gray{\pm0.006}}$ & $12.62_{\gray{\pm0.25}}$ & $87.9_{\gray{\pm0.9}}$ \\
\, + SC$_{0.5}$ & $1.015_{\gray{\pm0.024}}$ & $11.33_{\gray{\pm0.21}}$ & $\textbf{92.5}_{\gray{\pm1.6}}$ \\
\, + QS$_{0.4}$ & $1.045_{\gray{\pm0.009}}$ & $12.04_{\gray{\pm0.25}}$ & $87.1_{\gray{\pm1.3}}$ \\
\, + QS$_{0.4}$ + SC$_{0.5}$ & $\textbf{1.004}_{\gray{\pm0.026}}$ & $\textbf{10.89}_{\gray{\pm0.23}}$ & $91.5_{\gray{\pm1.5}}$ \\
\bottomrule
\end{tabular}
\caption{Impact of q-sampling (QS) and self-conditioning (SC) on SwissProt protein generation. SC trained with probability $p$ is denoted SC$_{p}$; QS with optimal activation timestep $t_{\mathrm{act}}$ is denoted QS$_{t_{\mathrm{act}}/T}$.}
\label{tab:proteins}
\vspace{-1.5em}
\end{wraptable}

In this section, we demonstrate that the quality improvements from q-sampling and self-conditioning are not specific to text, by conducting experiments on protein generation, a distinct discrete domain.
We use the SwissProt subset of the UniProt database~\citep{uniprot}.

\paragraph{Model}
We use the DiMA-$35$M model for protein sequence generation~\citep{dima} trained on shallow ESM-2 $150$M embeddings~\citep{esm2-protein}, following the training protocol described in the original paper.
We refer to this model as DiMA Emb, for consistency with TEncDM Emb.
We use $100$ steps during sampling. 

\paragraph{Metrics}
As a sequence quality metric, we use \textbf{pseudoperplexity}~\citep{pppl} (ESM-2 pppl) computed with the ESM-2 $650$M encoder transformer-based language model~\citep{esm2-protein}, and \textbf{cluster density}~\citep{dima} with an identity threshold of $t = 50\%$ (CD$_{0.5}$) as a diversity measure.
For consistency with the text generation experiments, we report $100 \cdot \mathrm{CD}_{0.5}$.
To assess the overall similarity between generated and reference protein distributions, we use \textbf{Fr\'{e}chet Distance} (FD-seq) computed on top of ProtT5 sequence representations~\citep{prot_t5}. 

\begin{figure}
    \centering
    \newcolumntype{M}[1]{>{\centering\arraybackslash}m{#1}}
    \setlength{\tabcolsep}{1pt}
    \begin{tabular}{M{0.31\linewidth}M{0.31\linewidth}}
        \midrule
         \multicolumn{2}{M{0.62\linewidth}}{\textbf{SwissProt protein generation}} \\
         \midrule
         No SC & With SC \\
         \includegraphics[width=0.31\textwidth]{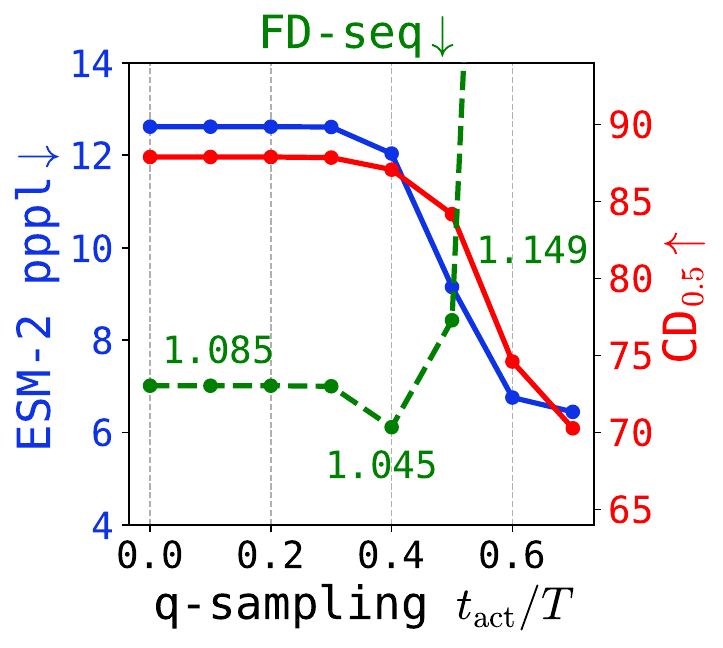} &
         \includegraphics[width=0.31\textwidth]{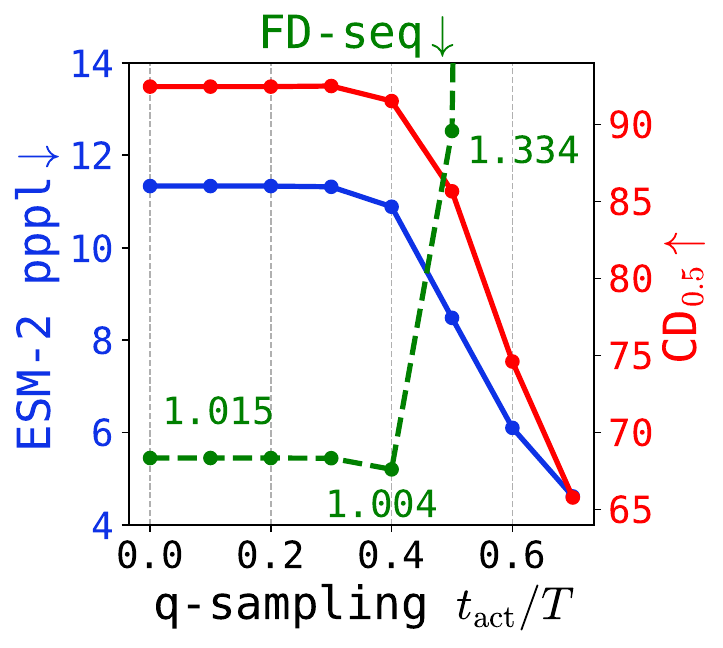} \\
    \end{tabular}
    \caption{Quality metrics with and without self-conditoning (SC) for different q-sampling activation timestep on SwissProt protein generation. We use the baseline SC model with $p=0.5$.}
    \label{fig:q_sampling_protein}
\end{figure}

\paragraph{Results}
Table~\ref{tab:proteins} compares baseline DDPM, self-conditioning, and q-sampling for protein generation.
Similarly to text experiments, both techniques reduce pseudoperplexity, but their effects on CD$_{0.5}$ differ: q-sampling decreases diversity while SC tends to increase it.
Both approaches reduce FD-seq, though the improvement from SC is substantial while q-sampling yields only a marginal gain.
Figure~\ref{fig:q_sampling_protein} further illustrates how quality metrics vary with $t_{\mathrm{act}}$, showing a pattern consistent with text generation: activating q-sampling earlier simultaneously reduces pseudoperplexity and diversity, with the optimal FD-seq achieved at an intermediate value of $t_{\mathrm{act}}$.

%% file: chapters/appendix/conditional.tex
\section{Additional results for conditional generation}\label{app:cond_generation}

This appendix presents additional results for sequence-to-sequence tasks.
Depending on the dataset, we measure BLEU~\citep{bleu}, ROUGE-1/2/L~\citep{rouge}, BertScore~\citep{bertscore}, or CodeBertScore~\citep{codebertscore}.
Tables~\ref{tab:q-sampling_sc_cond} and~\ref{tab:conala} report the performance of SC and q-sampling with optimal $t_{\mathrm{act}}$ across all four datasets, while Figure~\ref{fig:bleu_iwslt_qqp} shows BLEU as a function of $t_{\mathrm{act}}$ for IWSLT14 (machine translation) and QQP (question paraphrasing).
For IWSLT14, both results are negative: SC degrades performance, and pure DDPM outperforms q-sampling at any activation timestep.
For QQP, SC yields a statistically significant but modest improvement, while BLEU remains nearly constant across $t_{\mathrm{act}}$.
These results point to a qualitative difference from XSum and CoNaLa.

We hypothesize that the key factor is the strength of the conditioning signal.
In some equence-to-sequence tasks, the input makes the output nearly deterministic: in machine translation, for instance, valid outputs differ only in word order or synonym choice.
In others, such as code generation or summarization, the output space is much larger.
For example, in code generation, the same functionality can be implemented with different algorithms, reordered code blocks, or renamed variables.
In the strong-conditioning regime, the input already guides sampling away from low-density regions, preventing the failure mode that is addressed by SC and q-sampling.
Moreover, both techniques deviate the generation trajectory from the forward process, a deviation that proves beneficial in unconditional generation but harmful under strong conditioning.
This is particularly true for q-sampling, which does not replicate the forward process by design, as discussed in Section~\ref{sec:mitigation}.
In practice, determining whether the conditioning signal is strong enough may not always be straightforward.
When strong conditioning is suspected, it may be preferable to first evaluate the model without SC and with pure DDPM, since applying SC without necessity increases both training and inference time and may even degrade generation quality.

\begin{table}[t]
%\fontsize{8}{9}\selectfont
\small
\centering
\setlength{\tabcolsep}{3.4pt}
\begin{tabular}{l|cc|ccc|cc}
\toprule
& \multicolumn{2}{c|}{\textbf{IWSLT14} (SeqDiffuSeq)} & \multicolumn{3}{c|}{\textbf{QQP} (SeqDiffuSeq)} & \multicolumn{2}{c}{\textbf{XSum} (TEncDM Emb)} \\
\textbf{Solver} & \textbf{BLEU} $\uparrow$ & \textbf{R-1/2/L} $\uparrow$ & \textbf{BS} $\uparrow$ & \textbf{BLEU} $\uparrow$ & \textbf{R-1/2/L} $\uparrow$ & \textbf{BS} $\uparrow$ & \textbf{R-1/2/L} $\uparrow$ \\
\midrule
DDPM & 
$\textbf{29.19}$ & $\textbf{57.4} / \textbf{32.3} / \textbf{54.3}$ &
$\textbf{82.8}$ & $33.58$ & $\textbf{60.2}/\textbf{35.7}/\textbf{57.8}$ &
$59.3$ & $27.2/6.7/20.7$ \\
\, + SC$_{0.5}$ &
$28.93$ & $56.8/31.8/53.7$ &
$\textbf{82.8}$ & $\textbf{34.13}$ & $60.1/35.6/57.5$ & 
$66.3$ & $30.5/9.1/23.4$
\\
\, + QS$_{1.0}$ & 
$28.60$ & $56.4/31.7/53.5$ & 
$82.1$ & $33.50$ & $59.6/35.1/57.0$ & 
$63.4$ & $30.1/8.4/23.1$
\\
\, + QS$_{1.0}$ + SC$_{0.5}$ & 
$28.79$ & $56.8/31.9/53.7$ & 
$82.7$ & $\textbf{34.11}$ & $59.8/35.4/57.3$ & 
$\textbf{67.5}$ & $\textbf{31.9}/\textbf{9.9}/\textbf{24.5}$
\\
\bottomrule
\end{tabular}
\caption{Impact of q-sampling (QS) and self-conditioning (SC) on IWSLT14, QQP, and XSum with their corresponding models. SC trained with probability $p$ is denoted SC$_{p}$; QS with optimal activation timestep $t_{\mathrm{act}}$ is denoted QS$_{t_{\mathrm{act}}/T}$. For IWSLT-14 and QQP, q-sampling does not give a significant improvement for any $t_{\mathrm{act}}$, thus we report results for pure q-sampling $t_{\mathrm{act}}/T=1.0$, which is also optimal for XSum.} 
\label{tab:q-sampling_sc_cond}
\end{table}

\begin{figure}[t]
\centering
\begin{minipage}{0.35\textwidth}
    \centering
    \small
    \begin{tabular}{l|c}
        \toprule
        \textbf{CoNaLa} & \textbf{CodeBS} $\uparrow$ \\
        \midrule
        GENIE & $66.35$ \\
        \, + SC$_{0.5}$ & $67.34$ \\
        \, + QS$_{0.5}$ & $67.70$ \\
        \, + QS$_{0.5}$ + SC$_{0.5}$ & $\textbf{68.10}$ \\
        \bottomrule
    \end{tabular}
    \captionof{table}{Impact of q-sampling (QS) and self-conditioning (SC) on CoNaLa. SC trained with probability $p$ is denoted SC$_{p}$; QS with optimal activation timestep $t_{\mathrm{act}}$ is denoted QS$_{t_{\mathrm{act}}/T}$.}
    \label{tab:conala}
\end{minipage}
\hfill
\begin{minipage}{0.60\textwidth}
    \centering
    \includegraphics[width=0.495\textwidth]{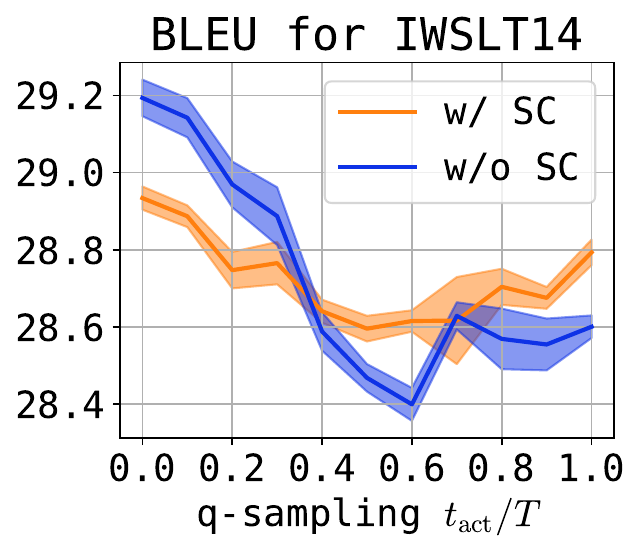}
    \includegraphics[width=0.495\textwidth]{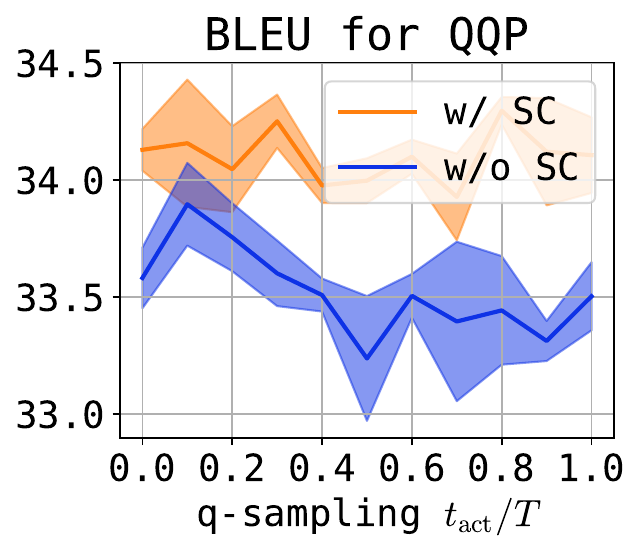}
    \caption{IWSLT-14 (left) and QQP (right) quality for varying q-sampling $t_{\mathrm{act}}$ with and without SC.}
    \label{fig:bleu_iwslt_qqp}
\end{minipage}
\end{figure}

%% file: chapters/appendix/image_examples.tex
\section{Image diffusion example}\label{app:image_diffusion}

In the main part of the paper, we identify a failure mode of Gaussian diffusion models that is specific to discrete data. 
To verify that this effect is indeed tied to discreteness, in this appendix we consider a continuous domain and study image generation on CIFAR-10 \citep{krizhevsky2009learning}. 
In the continuous case, data lie on a manifold, and the corresponding density transitions smoothly to a standard Gaussian, so the same failure mode is not expected to arise.

\paragraph{Experimental setup}

For this experiment, we train a Song U-Net~\citep{song2020score, karras2022elucidating} with approximately $62$M parameters on CIFAR-10 using the variance-preserving scheme and standard DDPM generation with $800$ sampling steps. 
To assess the generated distribution, we report FID~\citep{heusel2017gans} together with Precision, Recall, Density, and Coverage~\citep{naeem2020reliable}, which allows us to evaluate both generation quality and sample diversity. 
Precision, Recall, Density, and Coverage (PRDC) are computed in a fixed feature space by first extracting embeddings for real and generated images and then defining a local neighborhood radius for each sample via its $k$-nearest-neighbor distance within the same set. 
\begin{itemize}
  \item \textbf{Precision} is the fraction of generated samples that lie inside at least one real-sample neighborhood.
  \item \textbf{Recall} is the fraction of real samples that lie inside at least one generated-sample neighborhood.
  \item \textbf{Density} quantifies how many generated samples populate each real neighborhood on average.
  \item \textbf{Coverage} is the fraction of real samples for which at least one generated neighbor falls within the corresponding real neighborhood radius.
\end{itemize} 

Overall, FID and Precision primarily capture sample fidelity, while Recall and Coverage primarily capture diversity and support coverage. 
This choice is consistent with the methodology of the paper, where the effect of different heuristics is interpreted through the trade-off between quality and diversity rather than through a single metric alone.

\paragraph{Results}
We then evaluate the same heuristics as in the discrete setting: self-conditioning, q-sampling, and Minimum Bayes-Risk (MBR). 
In contrast to text and proteins, the overall picture on CIFAR-10 is qualitatively different: q-sampling degrades the quality of generated images and leads to visibly smoother samples. 
This observation is consistent with our analysis from the main text: q-sampling does not reproduce the trajectory of the reverse process, and it has smaller marginal variance that of DDPM (see Appendix.~\ref{app:prediction_variance}). 
In the discrete setting, this can be beneficial because it helps trajectories avoid low-density regions within the critical interval; however, in a continuous image domain, where this failure mode is absent, the same contraction becomes harmful. 

\begin{table}[t]
\centering
\small
\setlength{\tabcolsep}{4pt}
\caption{Impact of q-sampling (QS), self-conditioning, and MBR on image diffusion on CIFAR-10.}
\label{tab:cifar10}
\begin{tabular}{lccccc}
\toprule
\textbf{Method} & \textbf{FID} $\downarrow$ & \textbf{Precision} $\uparrow$ & \textbf{Recall} $\uparrow$ & \textbf{Density} $\uparrow$ & \textbf{Coverage} $\uparrow$ \\
\midrule
SongUNet & $8.47$ & $0.811$ & $0.673$ & $1.127$ & $0.956$ \\
+ QS$_{0.025}$  & $37.40$ & $0.693$ & $0.434$ & $0.713$ & $0.630$ \\
+ QS$_{0.05}$  & $56.72$ & $0.661$ & $0.251$ & $0.581$ & $0.455$ \\
+ QS$_{0.075}$ & $64.22$ & $0.647$ & $0.195$ & $0.557$ & $0.412$ \\
+ QS$_{0.1}$ & $70.07$ & $0.630$ & $0.172$ & $0.529$ & $0.392$ \\
+ SC & $16.48$ & $0.853$ & $0.565$ & $1.467$ & $0.922$ \\
+ MBR$_{0.1}$ & $9.42$ & $0.886$ & $0.568$ & $1.606$ & $0.974$ \\
+ MBR$_{0.2}$ & $10.17$ & $0.907$ & $0.546$ & $1.772$ & $0.976$ \\
+ MBR$_{0.3}$ & $11.15$ & $0.927$ & $0.507$ & $1.894$ & $0.980$ \\
\bottomrule
\end{tabular}
\end{table}

\begin{figure*}[t]
    \centering
    \setlength{\tabcolsep}{1pt}
    \renewcommand{\arraystretch}{1.0}
    \begin{tabular}{c c c c c c c}
        \includegraphics[width=0.13\linewidth]{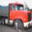} &
        \includegraphics[width=0.13\linewidth]{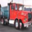} &
        \includegraphics[width=0.13\linewidth]{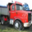} &
        \includegraphics[width=0.13\linewidth]{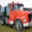} &
        \includegraphics[width=0.13\linewidth]{figures/cifar10/000000_m003.png} &
        \includegraphics[width=0.13\linewidth]{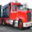} &
        \includegraphics[width=0.13\linewidth]{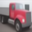} \\
        \includegraphics[width=0.13\linewidth]{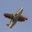} &
        \includegraphics[width=0.13\linewidth]{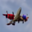} &
        \includegraphics[width=0.13\linewidth]{figures/cifar10/000001_m000.png} &
        \includegraphics[width=0.13\linewidth]{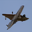} &
        \includegraphics[width=0.13\linewidth]{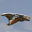} &
        \includegraphics[width=0.13\linewidth]{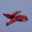} &
        \includegraphics[width=0.13\linewidth]{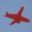} \\
        \small MBR$_{0.3}$ &
        \small Cand.\ \#1 &
        \small Cand.\ \#2 &
        \small Cand.\ \#3 &
        \small Cand.\ \#4 &
        \small Cand.\ \#5 &
        \small QS$_{0.1}$
    \end{tabular}
    \caption{Qualitative comparison on CIFAR-10. Generation is initialized from the same seeds. The first column corresponds to MBR-selected samples, while the last column corresponds to q-sampling (QS). As can be seen, q-sampling produces significantly smoother images.}
    \label{fig:cifar10_qualitative}
\end{figure*}

Self-conditioning can improve precision-oriented metrics by making the model’s predictions more confident, but this comes at the cost of diversity, as reflected in lower recall and coverage.
In other words, the generated distribution becomes narrower: samples are more concentrated around typical modes, while a smaller fraction of the data manifold is covered.
We observe a similar effect for MBR.
By selecting the most central sample among several candidates, MBR favors more typical generations and suppresses outliers.
In our implementation, MBR uses the $\ell_2$ distance in the InceptionV3 \citep{szegedy2016rethinking} latent space, i.e. the same representation used for FID evaluation.

Overall, these results support our main claim.
The improvements brought by self-conditioning, q-sampling, and MBR in text diffusion are not generic properties of these heuristics.
Rather, they arise because these methods mitigate the specific DDPM pitfalls caused by the multimodal structure of noisified discrete data.
On CIFAR-10, where the density transitions smoothly and sampling trajectories are not trapped between discrete modes, the same heuristics do not provide the same benefit and mainly induce the usual quality-diversity trade-off.

\section{Image toy example}\label{app:image_toy}

To provide additional evidence for our hypothesis that the heuristics considered in this work improve generation in discrete domains such as text, but do not provide the same benefit for images, we introduce an intermediate experimental domain.
The goal is to construct a task in which the model still operates on continuous images, while the data distribution has an internal discrete structure.
This allows us to isolate the role of discreteness without changing the diffusion parameterization or the form of the model output.

\paragraph{Experimental setup}

To this end, we select a small subset $S$ of CIFAR-10 images.
Each $32 \times 32$ image is split into four $16 \times 16$ patches, which are stored in a patch bank.
New images are then formed by composing patches from this bank.
In this way, the model still predicts images in the continuous pixel space, but the set of valid outputs is determined by a finite collection of patch combinations.
Thus, the model must memorize the atomic patches themselves while also generalizing over the discrete structure induced by their composition.

We consider two variants of this task.
In the first one, we study a conditional setting.
We train a conditional diffusion model in which the conditioning variable specifies an unordered set of patch identities contained in the image.
Because the order of the patches is not provided, the same conditioning label may correspond to multiple valid outputs that differ by spatial permutation.
This creates a multimodal conditional distribution while keeping the output space continuous.
For the task we set $|S| = 2$, which yields a population of $8^4 = 4096$ possible objects.

In the second variant, we introduce transition rules between patches.
The first patch, corresponding to the top-left position, is sampled uniformly from the full patch bank.
Each subsequent patch is then sampled uniformly from a restricted subset determined by a transition graph over patch identities. 
In this case, the total number of valid images is $4|S| \cdot r^3$, where $r$ denotes the number of allowed transitions per step.
The purpose of this construction is to introduce a more complex discrete structure and prevent the model from simply memorizing the full set of training images.
For this task we set $|S| = 10, r \in \{4, 5\}$ and for both tasks we limit train size to $1000$ images.

\paragraph{Metrics}

For evaluation, we introduce two main metrics: correctness and diversity.
Correctness is defined as the fraction of valid generated objects, where validity means consistency with the corresponding conditional label in the conditional setup, or with the transition rules in the transition-rule setup.
Diversity is defined as the fraction of unique correct generated objects that do not belong to the training set, measured with respect to the full population of valid objects.
In this way, correctness captures whether the model follows the discrete structure of the task, while diversity measures its ability to generate valid but non-memorized combinations.

To compute these metrics, each generated image is first split into four patches. 
For every patch, we then find the closest patch from the training bank in $\ell_2$ distance and assign the corresponding patch index. 
In this way, each generated image is mapped to an ordered vector of four discrete patch identities, after which correctness and diversity can be computed directly in the discrete space.

\begin{figure}[t]
  \centering
  \includegraphics[width=\linewidth]{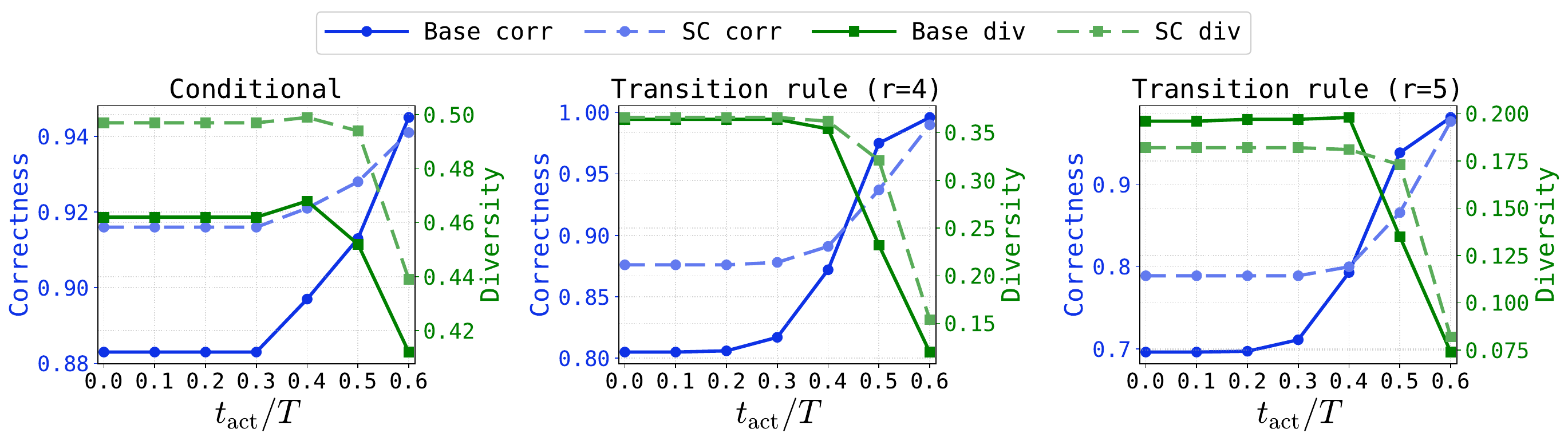}
  \caption{Correctness (left axis, blue) and diversity (right axis, green) as functions of the q-sampling activation time $t_{\mathrm{act}}$, shown for the three patch-based CIFAR-10 experiments. Solid lines correspond to the base model, dashed lines to self-conditioning (SC).}
  \label{fig:qs_tradeoff_all_horizontal}
\end{figure}

\begin{figure}[t]
  \centering
  \setlength{\tabcolsep}{6pt}
  \begin{tabular}{cc}
    \includegraphics[width=0.18\linewidth]{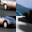} &
    \includegraphics[width=0.18\linewidth]{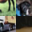}
  \end{tabular}
  \caption{Examples of patch-based training images.}
  \label{fig:patch_train_examples}
\end{figure}

\paragraph{Results}

We then analyze the effect of q-sampling and self-conditioning on generation. 
In both experimental setups, q-sampling with different values of $t_{\mathrm{act}}$ acts as a strong control parameter for the correctness-diversity trade-off, consistent with the results in the main part of the paper and with our theoretical interpretation. 
In particular, earlier activation of q-sampling leads to a substantial reduction in diversity. 
At the same time, the effect becomes pronounced only for $t_{\mathrm{act}}/T$ around $0.3$ and higher. 
This is explained by the fact that the discrete part of the distribution is already largely formed during the earlier stages of sampling, whereas the later steps mostly refine the continuous appearance of the patches. 
As a result, activating q-sampling too late has little influence on the discrete structure of generated objects, while activating it earlier causes the trajectory to avoid a substantial fraction of valid discrete modes, leading to lower diversity.

Self-conditioning shows a different behavior. 
It improves correctness without a comparable drop in diversity, which is consistent with our main observation that self-conditioning is particularly useful in discrete domains. 
In addition, when q-sampling is applied on top of a self-conditioned model, it again provides a mechanism for controlling the correctness-diversity trade-off.
Thus, these experiments with synthetic discrete data embedded in a continuous space further support our interpretation of the problems arising in text diffusion and show that the proposed heuristics are indeed useful in discrete domains.